\documentclass{article}

\usepackage[preprint]{neurips_2026}

\usepackage{amsmath,amsfonts,bm}

\def\eqref#1{equation~\ref{#1}}

\def\1{\bm{1}}

\DeclareMathAlphabet{\mathsfit}{\encodingdefault}{\sfdefault}{m}{sl}
\SetMathAlphabet{\mathsfit}{bold}{\encodingdefault}{\sfdefault}{bx}{n}

\usepackage[utf8]{inputenc} 
\usepackage[T1]{fontenc}    
\usepackage{hyperref}       
\usepackage{url}            
\usepackage{booktabs}       
\usepackage{amsfonts}       
\usepackage{nicefrac}       
\usepackage{microtype}      
\usepackage{xcolor}         
\usepackage{tabularx}
\usepackage{booktabs}
\usepackage{algorithm}
\usepackage{mathtools}
\usepackage{multirow}
\usepackage{chngcntr}
\usepackage{amsthm}
\usepackage{algpseudocode}
\usepackage{wrapfig}
\usepackage{lipsum}
\usepackage{subcaption}

\newtheorem{lemma}{Lemma}
\newtheorem{corollary}{Corollary}
\usepackage{algorithm}
\usepackage{algorithmicx,algcompatible}
\usepackage{fontawesome5}
\usepackage{amsmath}
\usepackage{amssymb}
\usepackage{tabularx, booktabs}
\usepackage{enumitem}
\newtheorem*{proposition}{Proposition}
\theoremstyle{remark}
\newtheorem{remark}{Remark}
\usepackage{pifont}
\newcommand{\xmark}{\ding{55}}
\newcommand{\cmark}{\ding{51}}

\title{SurF: A Generative Model for Multivariate Irregular Time Series Forecasting}

\author{%
  Mohammad R. Rezaei \\
  Department of Computer Science\\
  University of Toronto\\
  Vector Institute\\
  Toronto, ON, Canada \\
  \texttt{mr.rezaei@mail.utoronto.ca} \\
  \And
  Tejas Balaji \\
  Department of Computer Science\\
  University of Toronto\\
  Vector Institute\\
  Toronto, ON, Canada \\
  \texttt{tejas.balaji@mail.utoronto.ca} \\
  \And
  Rahul G. Krishnan \\
  Department of Computer Science\\
  University of Toronto\\
  Vector Institute\\
  Toronto, ON, Canada \\
  \texttt{rahulgk@cs.toronto.edu} \\
}

\begin{document}

\maketitle

\newtheorem{theorem}{Theorem}
\newtheorem{definition}{Definition}
\begin{abstract}
Irregularly sampled multivariate event streams remain a difficult modality for generative modeling: tokenization-based approaches break down when inter-event intervals vary by orders of magnitude. We (i) propose \textbf{SurF}, a generative model that uses the Time Rescaling Theorem (TRT) as a learnable bijection between event sequences and i.i.d.\ unit-rate exponential noise, enabling a single model to be trained across heterogeneous event-stream datasets; (ii) three efficient parameterizations of the cumulative intensity that scale to long sequences; and (iii) a Transformer-based encoder for multi-dataset pretraining. On six real-world benchmarks, SurF achieves the best reported time RMSE on Earthquake, Retweet, and Taobao, and is within trial-level noise of the strongest specialist on the remaining three. Under a strict leave-one-out protocol, the held-out checkpoint beats every classical and neural-autoregressive baseline on $5/6$ datasets and beats every baseline on Amazon and Earthquake, an initial step toward foundation models over asynchronous event streams.
\vspace{0.5em}
\noindent\textbf{Code:} \href{https://github.com/MrRezaeiUofT/SurF}{\faGithub\ \texttt{github.com/MrRezaeiUofT/SurF}}.
\end{abstract}

\section{Introduction}

\begin{figure*}[ht]
\centering
\includegraphics[width=.8\textwidth]{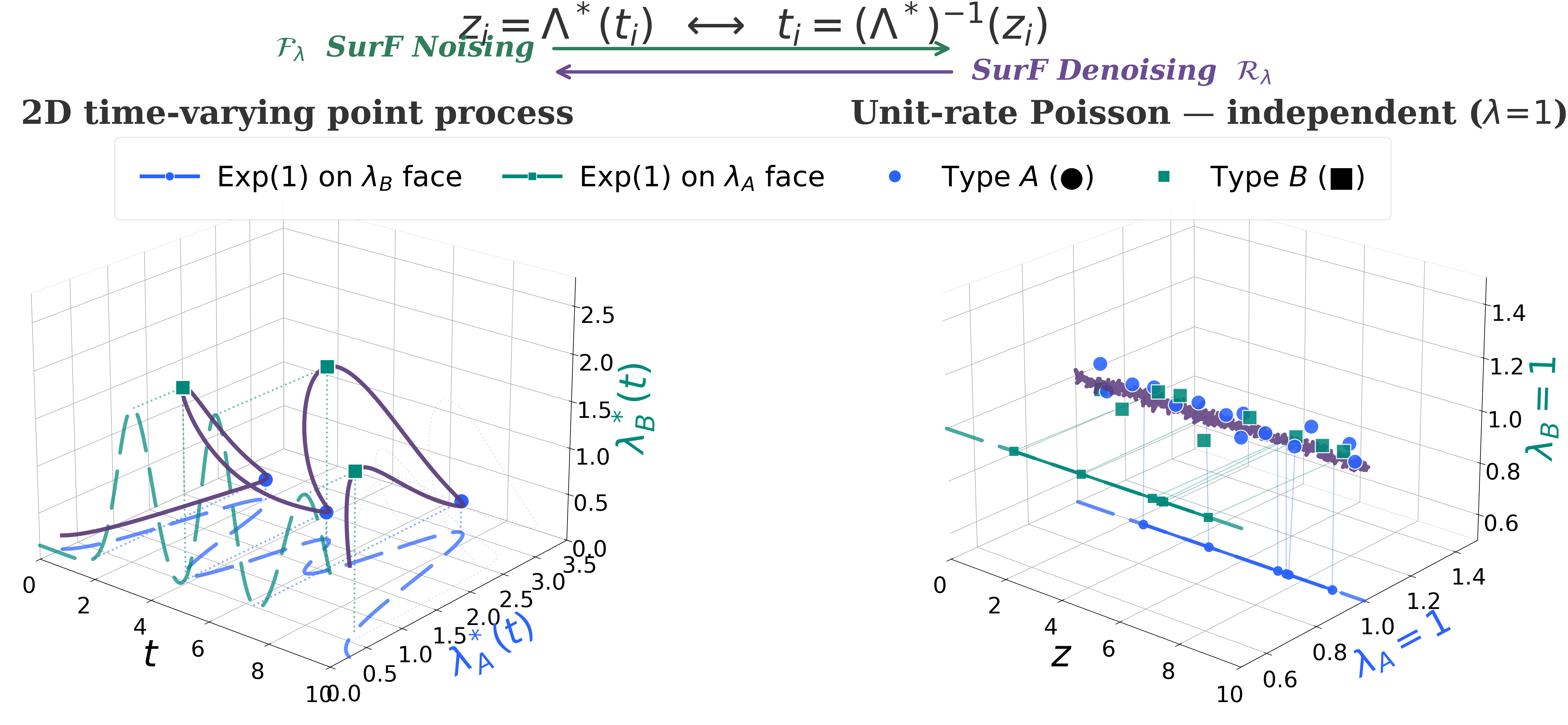}
\vspace{-1.em}
\caption{\small
\textbf{The SurF noising--denoising framework in the joint intensity space.}
A 2D point process with two anti-correlated event types.
\textbf{(Left)}~In the original time span, type~$A$ activity suppresses
type~$B$ and vice versa.
\textbf{(Right)}~After the SurF noising process $\mathcal{F}_\lambda$
(Theorem~\ref{thm:time_rescaling}), the trajectory collapses to i.i.d.\
$\mathrm{Exp}(1)$ inter-arrival times along the rescaled axis~$z$;
the inverse $\mathcal{R}_\lambda$ recovers the original dynamics losslessly.}
\label{fig:3d_overview}
\vspace{-0.5em}
\end{figure*}

Generative models of regularly sampled time-series data such as TimesFM~\citep{das2024decoder}, Chronos~\citep{ansari2024chronos}, and Lag-Llama~\citep{rasul2024lag} share a key assumption: that data arrive at a fixed rate along synchronized channels. This precludes their application to \emph{irregularly sampled multivariate event streams},
where events arrive asynchronously, streams update at different frequencies, and observation timing itself carries semantic meaning; existing tokenization schemes also break down when inter-event intervals vary by orders of magnitude. Yet such data arise throughout high-stakes domains: clinical records with variable vital-sign intervals~\citep{rubanova2019latent,rezaei2022direct}, asynchronous financial order flows~\citep{bacry2015hawkes,hawkes2018hawkes},
event-driven sensor measurements~\citep{shukla2021multi}, and
coordinated neural spike trains~\citep{truccolo2005point,pillow2008spatio}. The standard mathematical object for these data is the \emph{temporal point process} (TPP), which models a stream of events through its
conditional intensity $\lambda^*(t \mid \mathcal{H}_t)$; the instantaneous event rate at time $t$ given history $\mathcal{H}_t$. Its
integral, the cumulative intensity
$\Lambda^*(t) = \int_0^t \lambda^*(s \mid \mathcal{H}_s)\,ds$, governs how long one expects to wait for the next event. Fitting a TPP amounts
to choosing a parameterization of $\lambda^*$ and estimating its parameters from data.

A recurring difficulty in this fitting problem is evaluation: standard regression diagnostics do not apply to event data, and since every
process has its own intensity, timescales, and shape, there is no obvious reference distribution to compare a fitted model against. In the early 2000s, the
Time Rescaling Theorem (TRT)~\citep{brown2002time} resolved this: if the fitted $\Lambda$ is correct, the rescaled gaps
$\Delta z_i = \Lambda(t_i) - \Lambda(t_{i-1})$ are i.i.d.\
$\mathrm{Exp}(1)$, so a Q--Q plot or Kolmogorov--Smirnov
statistic~\citep{truccolo2005point,brown2002time} suffices for goodness-of-fit. The TRT became standard practice as an
\emph{evaluation} tool. Two decades later, neural TPPs rose in popularity as a means to leverage the representational flexibility of neural networks for event-stream modeling. Most neural TPPs are fit by maximum likelihood, which at every step
requires the intensity $\lambda_\theta$ at each event and its integral
$\int_{0}^{T}\lambda_\theta(s)\,ds$ over the observation window.
Intensity-first methods parameterize $\lambda_\theta$
directly and approximate the integral by numerical quadrature, a cost
that grows with the window and accumulates with every gradient step.
The cumulative-hazard family~\citep{omi2019fully,shchur2019intensity}
avoids this by parameterizing $\Lambda_\theta$ with a monotone neural
network and recovering $\lambda_\theta$ by differentiation, collapsing
the window integral to a single evaluation $\Lambda_\theta(t)$.

Our key insight is that this cumulative-hazard family is, in effect,
fitting only the forward direction of the TRT. SurF treats $\Lambda$ as
a bijection and uses it in both directions, enabling a generative
model: $\Lambda_\theta$ maps event sequences to i.i.d.\
$\mathrm{Exp}(1)$ noise (forward, for training), while
$(\Lambda_\theta)^{-1}$ maps exponential noise back to event times
(reverse, for sampling). Prior neural TPPs are trained per-dataset; concatenating event streams
and summing their likelihoods. However, this approach can lead to instabilities in learning 
since per-dataset losses have incompatible scales. SurF resolves
this by recognizing that we can learn patterns in multiple event streams using a single shared canonical target i.e. $\Lambda_\theta$ fit
jointly on many streams optimizes a single well-defined objective and pushes every dataset onto $\mathrm{Exp}(1)$. Since all sequences map to the same target, the
shared $(\phi_\theta,\Lambda_\theta)$ accumulates similarities across event streams in a general manner enabling the same parameters
to successfully enable zero-shot prediction to a new (unseen) dataset (Section~\ref{sec:domain_invariance}).
\textbf{Contributions.}
\textbf{(1) TRT as an explicit bidirectional flow.} Prior cumulative-hazard neural TPPs~\citep{omi2019fully,shchur2019intensity}
invoke the TRT implicitly at training time. We make both directions
explicit and record the conditions on $\Lambda_\theta$ (smooth,
strictly increasing on $\mathbb{R}^+$) under which the inverse function
theorem yields a smooth bijection
(Theorem~\ref{thm:reverse_rescaling}); our contribution lies in the
framing and its consequences. \textbf{(2) Three parameterizations of
the intensity function.} We parameterize $\Lambda_\theta$ as a monotone
neural network. MoE and CSB are fully closed-form; GLQ uses an
unconstrained positive MLP with a fixed $O(Q)$ per-interval rule whose
cost is $T$-independent and whose error is negligible in practice
(Appendix~\ref{app:variants_quadrature}). \textbf{(3) A shared target for multi-corpus training.} Because the TRT guarantees that \emph{every} process has an $\mathrm{Exp}(1)$ compensator, $\Lambda_\theta$ fitted jointly on many corpora optimizes a single loss rather than a sum of per-dataset likelihoods on different scales. The target is the same for every dataset. Prior cumulative-hazard and intensity-free models~\citep{omi2019fully,shchur2019intensity} fit dataset-specific latent variables with no space shared across corpora, and are trained one dataset at a time. \textbf{(4) cross-dataset and
zero-shot evaluation.} On six benchmarks, a jointly-trained SurF checkpoint achieves the best reported time RMSE on Earthquake, Retweet, and Taobao, and is within trial-level noise of the strongest baseline on the remaining three. Under a strict leave-one-out protocol, the held-out checkpoint beats every classical and neural-autoregressive baseline on $5/6$ datasets (losing only StackOverflow); to our knowledge the first reported leave-one-out cross-dataset evaluation for a learned cumulative intensity.
\section{Background}
Let $\mathcal{D}=\{(t_i,k_i)\}_{i=1}^N$ denote a sequence of $N$ events
in $[0,T]$ with $0<t_1<\cdots<t_N\leq T$ and types
$k_i\in\{1,\ldots,K\}$. A temporal point process is specified by its
conditional intensity $\lambda^*(t,k\mid\mathcal{H}_t)$, where
$\mathcal{H}_t=\{(t_j,k_j):t_j<t\}$~\citep{daley2003introduction}.
Writing the ground intensity
$\lambda^*(t\mid\mathcal{H}_t)=\sum_k\lambda^*(t,k\mid\mathcal{H}_t)$,
the likelihood factorizes as \begin{equation}\label{eq:ppp_likelihood}
p(\mathcal{D}) = \Bigl[\prod_{i=1}^N \lambda^*(t_i,k_i\mid\mathcal{H}_{t_i})\Bigr]
\exp\!\Bigl(-\!\int_0^T \lambda^*(s\mid\mathcal{H}_s)\,ds\Bigr),
\end{equation}
with the exponential term the survival probability over $[0,T]$. The
name \emph{Survival Flow} captures this structure: SurF learns a
bijective flow between observed event times and the canonical survival
form. Marks factorize as
$\lambda^*(t,k)=\lambda^*(t)\,p(k\mid t,\mathcal{H}_t)$, which decomposes
the joint NLL additively (Appendix~\ref{app:marks}); the rest of this
section develops the ground-intensity machinery.

\paragraph{Time Rescaling Theorem (TRT).}\label{sec:trt}
The TRT describes the existence of a common mapping that SurF will map event sequences onto using a learnable flow. 
\begin{theorem}[Time Rescaling~\citep{brown2002time}]\label{thm:time_rescaling}
Let $\{t_i\}_{i=1}^N$ be a realization of a point process with intensity
$\lambda^*(t\mid\mathcal{H}_t)>0$, and define the cumulative intensity
$\Lambda^*(t)=\int_0^t\lambda^*(s\mid\mathcal{H}_s)\,ds$, with
$\Lambda^*(t)\to\infty$ as $t\to\infty$. The
transformed times $z_i=\Lambda^*(t_i)$ form a unit-rate Poisson process;
equivalently, the gaps $\Delta z_i=z_i-z_{i-1}$ are i.i.d.\
$\mathrm{Exp}(1)$ (Appendix~\ref{app:trt-proof}).
\end{theorem}

The theorem maps \emph{any} temporal point process to the same canonical
distribution, $\mathrm{Exp}(1)$, enabling the creation of a flow via a
generative model. To be able to generate valid samples after learning, we need to be able to learn the reverse direction, which the next section establishes.
\section{SurF: Survival Flow for Temporal Point Processes}

While Theorem~\ref{thm:time_rescaling} establishes the existence of a forward map $t \mapsto \Lambda^*(t)$ that maps a point process to a unit-rate Poisson reference, three problems remain. First, it does not guarantee that $\Lambda^*$ is \emph{invertible}, so we cannot map noise back to event times; it does not establish \emph{smoothness} of the inverse, which we need for gradient-based training; and it does not provide the \emph{Jacobian}
required to evaluate likelihoods under a change of variables. We propose the following theorem, which closes all three gaps under a mild positivity condition on $\lambda^*$. This enables us to turn the TRT from an evaluation tool into a learnable flow.

\begin{theorem}[Reverse Rescaling and Bijectivity]\label{thm:reverse_rescaling}
Assume $\lambda^*(\cdot\mid\mathcal{H}_\cdot)$ is continuous and satisfies $\lambda^*(t\mid\mathcal{H}_t)\geq\lambda_{\min}>0$ for all $t\geq 0$. Then $\Lambda^*\!:\mathbb{R}^+\!\to\mathbb{R}^+$ is a $C^1$ bijection with $C^1$ inverse; if in addition $\lambda^*$ is $C^k$ ($k\geq 0$), then both
$\Lambda^*$ and $(\Lambda^*)^{-1}$ are $C^{k+1}$. Conversely, if $\{z_i\}$ is a realization of a unit-rate Poisson process, then $\{t_i=(\Lambda^*)^{-1}(z_i)\}$ is a realization of the original point
process with conditional intensity $\lambda^*$, and the change of variables $\boldsymbol{\tau}\mapsto\boldsymbol{\Delta z}$ has lower-triangular Jacobian
$J=\partial\boldsymbol{\Delta z}/\partial\boldsymbol{\tau}$ with $\det J =
\prod_{i=1}^N\lambda^*(t_i\mid\mathcal{H}_{t_i})>0$.
\end{theorem}

Proof sketch: $(\Lambda^*)'=\lambda^*\geq\lambda_{\min}>0$, so $\Lambda^*$
is strictly increasing and $C^1$, with $\Lambda^*(0)=0$ and
$\Lambda^*(t)\geq\lambda_{\min}\,t\to\infty$; together with monotonicity,
this gives a $C^1$ bijection $\mathbb{R}^+\!\to\mathbb{R}^+$, and the
inverse function theorem yields a $C^1$ inverse. The reverse direction is the contrapositive
of Theorem~\ref{thm:time_rescaling}: the rescaling
$\Phi:\{t_i\}\mapsto\{\Lambda^*(t_i)\}$ pushes the law of the original
process forward to unit-rate Poisson, and bijectivity of $\Phi$ makes
$\Phi^{-1}$ pull the unit-rate Poisson law back to the original (Appendix~\ref{app:trt-proof}).

Theorem~\ref{thm:reverse_rescaling} requires
$\lambda^*(t\mid\mathcal{H}_t)\geq\lambda_{\min}>0$.
We enforce this structurally with a small floor
$\lambda_{\text{floor}}>0$ (learnable, default $10^{-4}$; see
Appendix~\ref{app:floor}). This introduces an additive bias
of order $\tau_i\cdot\lambda_{\text{floor}}$ per inter-event interval,
bounded above by $T\cdot\lambda_{\text{floor}}$ per sequence (below the
NLL noise floor on all six benchmarks: $\leq\!0.1$ nats per sequence at
$\lambda_{\text{floor}}{=}10^{-4}$ and normalized $T\leq 10^3$), but
potentially material for processes with dead-time (neural
refractoriness, market closures), for which Appendix~\ref{app:floor}
provides a learnable-floor and explicit-masking scheme that recover
bijectivity without changing the Jacobian factorization. The two theorems define mutually inverse maps that together form the
SurF bijection (Figure~\ref{fig:3d_overview}). The forward map
$\mathcal{F}_\lambda(t)=\Lambda^*(t)$ warps time so that high-intensity
regions stretch and low-intensity regions compress, yielding i.i.d.\
$\mathrm{Exp}(1)$ gaps. The inverse
$\mathcal{R}_\lambda(z)=(\Lambda^*)^{-1}(z)$ reconstructs temporal
structure and can be obtained by solving
$dt/dz=1/\lambda^*(t\mid\mathcal{H}_t)$ with $t(0)=0$: high intensity
makes time advance slowly (event clusters), low intensity makes it
advance quickly (gaps). A detailed comparison to prior cumulative-hazard
neural TPPs (FullyNN~\citep{omi2019fully}, log-normal
mixture~\citep{shchur2019intensity}) is in
Appendix~\ref{app:related_cumhazard}.

\subsection{Training Objective: Amortized SurF Loss}\label{sec:training}

To learn the SurF bijection, we \emph{directly amortize the cumulative
intensity}: a neural network learns
$\Lambda_\theta(\Delta t\mid\mathbf{h})$, and the instantaneous intensity
is recovered by differentiation. This inverts the usual pipeline---
instead of \emph{model~$\lambda$ then integrate}, we \emph{model~$\Lambda$
then differentiate}; yielding a maximum-likelihood objective with no window-level integration over $[0,T]$. We parameterize $\Lambda_\theta:\mathbb{R}^+\!\times\mathbb{R}^d\!\to\mathbb{R}^+$ using a monotonically increasing function of $\Delta t=t-t_{i-1}$ conditioned on the
encoded history $\mathbf{h}_{i-1}$, with $\Lambda_\theta(0\mid\mathbf{h})=0$. The intensity is then
\begin{equation}\label{eq:amort_intensity}
\lambda^{\dagger}_\theta(\Delta t\mid\mathbf{h}) \;=\;
\frac{\partial \Lambda_\theta(\Delta t\mid\mathbf{h})}{\partial (\Delta t)}
\;=\; \lambda_{\text{floor}} + \lambda_\theta(\Delta t\mid\mathbf{h})
\;\geq\; \lambda_{\text{floor}} \;>\; 0,
\end{equation}
with positivity guaranteed by monotonicity and by the floor. Below, $\lambda_\theta$ denotes the unfloored variant-specific term and $\lambda^{\dagger}_\theta$ the quantity used in the objective and in our implementation; every variant carries $\lambda_{\text{floor}}\Delta t$ in $\Lambda_\theta$, so the identity above holds exactly.

\noindent\textbf{SurF-MoE} (mixture of exponentials, closed-form). With
$w_j,\gamma_j>0$ output by the history encoder, $
\lambda_\theta(\Delta t\mid\mathbf{h}) = \sum_{j=1}^{J} w_j e^{-\gamma_j \Delta t},
\Lambda_\theta(\Delta t\mid\mathbf{h}) = \lambda_{\text{floor}} \Delta t +\sum_{j=1}^{J} \tfrac{w_j}{\gamma_j}\bigl(1 - e^{-\gamma_j \Delta t}\bigr).$

\noindent\textbf{SurF-CSB} (cumulative softplus basis, closed-form).
$\Lambda_\theta$ is a sum of shifted softplus primitives with positive
$\alpha_m,\beta_m$,$
\Lambda_\theta(\Delta t\mid\mathbf{h}) = \lambda_{\text{floor}} \Delta t + \sum_{m=1}^{M} \alpha_m \bigl[\log(1+e^{\beta_m \Delta t + \delta_m}) - \log(1+e^{\delta_m})\bigr],$ whose derivative
$\lambda_\theta=\sum_m \alpha_m\beta_m\,\sigma(\beta_m\Delta t+\delta_m)$
is also closed-form.

\noindent\textbf{SurF-GLQ} (unconstrained positive MLP, per-interval
quadrature). $\lambda_\theta=\mathrm{softplus}(f_\theta(\Delta t,\mathbf{h}))$
for an unconstrained MLP $f_\theta$; $\Lambda_\theta$ is obtained by applying
fixed $Q$-point Gauss--Legendre quadrature to
$\lambda^{\dagger}_\theta=\lambda_{\text{floor}}+\lambda_\theta$ on $[0,\tau_i]$ with $Q=8$, so the floor is integrated exactly alongside the MLP term. This quadrature is (i) applied per inter-event interval rather than over $[0,T]$, (ii) of cost $O(NQ)$ with $Q$ fixed and independent of $T$, and (iii) empirically below floating-point noise at $Q=8$ on all six
benchmarks (Appendix~\ref{app:glq_error}), so it does not enter the training gradient in practice. This is the precise sense in which SurF avoids the $O(NQ')$, $T$-dependent quadrature over $[0,T]$ that bottlenecks THP/SAHP. Variant-selection guidelines are in Appendix~\ref{app:variants_quadrature}. Substituting~\eqref{eq:amort_intensity} into~\eqref{eq:ppp_likelihood} and applying the exact change of variables $\boldsymbol{\tau}\mapsto\boldsymbol{\Delta z}$ with the Jacobian of
Theorem~\ref{thm:reverse_rescaling} (Appendix~\ref{app:likelihood-derivation}) gives the amortized SurF
loss:
\begin{equation}\label{eq:amort_loss}
\mathcal{L}_{\text{SurF}}(\theta) \;=\;
\sum_{i=1}^N \Lambda_\theta(\tau_i\mid\mathbf{h}_{i-1})
+ \Lambda_\theta(\Delta T\mid\mathbf{h}_N)
- \sum_{i=1}^N \log \lambda^{\dagger}_\theta(\tau_i\mid\mathbf{h}_{i-1}),
\end{equation}
where $\tau_i=t_i-t_{i-1}$ and $\Delta T=T-t_N$. The first two terms are forward passes of the monotone network; the third is its derivative (autodiff or closed form). For marked sequences $\mathcal{L}_{\text{SurF}}$ is added to the mark cross-entropy $\mathcal{L}_{\text{type}}$ with scale-balancing weight $\beta_{\text{type}}$ (Appendix~\ref{app:marks}).

\textbf{Generative Sampling.}\label{sec:sampling}
SurF generates events by inverting the time rescaling: sample
exponential noise and solve for event times under the learned cumulative
intensity. Given history $\mathbf{h}$, draw $U\sim\mathrm{Uniform}(0,1)$,
set $z=-\log U$ so that $z\sim \textrm{Exp}(1)$, and solve $\Lambda_\theta(\Delta t\mid\mathbf{h})=z$
for $\Delta t$. Because $\Lambda_\theta$ is $C^1$, strictly increasing,
and lower-bounded by $\lambda_{\text{floor}}>0$, the root exists and is
unique. We use a safeguarded Newton
method~\citep{nocedal2006numerical}; the update $
\Delta t^{(n+1)} = \Delta t^{(n)}
- \frac{\Lambda_\theta(\Delta t^{(n)}\mid\mathbf{h})-z}{\lambda^{\dagger}_\theta(\Delta t^{(n)}\mid\mathbf{h})}
$ is accepted when it remains inside the current bracket, and bisection is
used otherwise. This converges globally from any valid bracket, with
local quadratic rate once inside the Newton--Kantorovich region
(Appendix~\ref{app:newton-convergence}). The mark is drawn from
the softmax of $\mathrm{TypeNet}_\theta(\mathbf{h})$ and the history is
updated; full procedure in Algorithm~\ref{alg:surf_sampling_newton}. Table~\ref{tab:method_comparison} summarizes the parameterization of
SurF against existing neural TPPs. Having established how $\Lambda_\theta$ is parameterized, we now ask whether a single model trained on multiple datasets can generalize to unseen ones at zero shot.
\begin{table}[ht]
\vspace{-1 em}
\centering
\setlength{\tabcolsep}{4pt}
\renewcommand{\arraystretch}{1.05}
\caption{\textbf{Parameterization comparison across neural TPPs.}
``Window quadrature'' refers to the $O(NQ')$ numerical integration of
$\lambda$ over $[0,T]$ inside the training likelihood;
``per-interval'' refers to a fixed $O(Q)$ rule on $[0,\tau_i]$ whose
cost does not grow with $T$. The Sampling column reports the procedure
used to generate event times given a learned model. $^{\ast}$FullyNN inverts its cumulative hazard at the single fixed quantile
$\Phi(\tau)=\log 2$ by bisection to obtain a median point prediction; a generative sampler over draws $z\sim\mathrm{Exp}(1)$ is not specified as we do in SurF. }
\label{tab:method_comparison}
\footnotesize
\resizebox{\columnwidth}{!}{%
\begin{tabular}{l|l|l|l|l}
\toprule
\textbf{Method} & \textbf{Parameterizes} & \textbf{Likelihood form} &
\textbf{Window quadrature} & \textbf{Sampling}  \\
\midrule
RMTPP~\citep{du-16-recurrent}    & $\lambda$ (RNN, exp. decay)            & Closed-form (parametric)        & None (parametric)        & Inverse CDF              \\
NHP~\citep{mei-17-neuralhawkes}  & $\lambda$ (CT-LSTM)                    & MC integral                     & Yes, MC over $[0,T]$     & Thinning                 \\
THP~\citep{zuo2020transformer}   & $\lambda$ (Transformer)                & MC / Simpson                    & Yes, $O(NQ')$            & Thinning                 \\
SAHP~\citep{zhang-2020-self}     & $\lambda$ (self-attention)             & MC                              & Yes, $O(NQ')$            & Thinning                 \\
IFTPP~\citep{shchur2019intensity}& Density (log-normal mixture)           & Closed-form                     & N/A (no $\lambda$)       & Sample density           \\
$^{\ast}$FullyNN~\citep{omi2019fully}     & $\Lambda$ (monotone FFN)               & Closed-form                     & None                     & Bisection at $\Phi=\log 2$           \\
\midrule
\textbf{SurF-MoE} (ours) & $\Lambda$ (mixture of exponentials)            & \textbf{Closed-form}            & \textbf{None}            & Newton on $\Lambda^{-1}$  \\
\textbf{SurF-CSB} (ours) & $\Lambda$ (cumulative softplus basis)          & \textbf{Closed-form}            & \textbf{None}            & Newton on $\Lambda^{-1}$  \\
\textbf{SurF-GLQ} (ours) & $\lambda=\mathrm{softplus}(f_\theta)$, $\Lambda$ via GLQ & Per-interval $O(Q)$, $T$-indep. & \textbf{None at window} & Newton on $\Lambda^{-1}$  \\
\bottomrule
\end{tabular}}
\end{table}
\subsection{Zero-Shot Capabilities Across Datasets}\label{sec:domain_invariance}
We now show that SurF's time-rescaling objective provides the ability to perform zero-shot prediction. To understand why, we examine this first through the lens of the history encoder $\phi_\theta$ (Appendix~\ref{app:architecture}). The goal is to learn an invertible mapping $\Lambda_\theta$ from the inter-arrival times of an arbitrary temporal point process to a sequence of Exp(1) distributed random variables (the existence of which is guaranteed by Theorem~\ref{thm:reverse_rescaling}). Since each $\Lambda_\theta(\tau_i\mid \mathbf{h}_{i-1})$ term of $\mathcal{L}_{\textrm{SurF}}$ is Exp(1) distributed at the well-specified optimum by Corollary~\ref{cor:time_rescaling_corollary}, training on any dataset results in minimizing the negative log-likelihood of an Exp(1) distribution governed by $\theta$---in this case, $\phi_\theta$. Since $\phi_\theta$ is shared across datasets, the variance in the inter-arrival times from the multiple training datasets helps create a robust mapping from input to Exp(1) space. Moreover, although multiple input TPPs might have similar inter-arrival times between certain events, the dependence of $\phi_\theta$ on the event history ensures that event patterns in input space play a pivotal role in governing the mapping. The signal
accumulates in the shared parameters $(\phi_\theta,\Lambda_\theta)$ as
a dataset-general repertoire of history features (clustering, decay,
mark--time coupling, refractory gaps); per-dataset rate and scale
constants are absorbed by the rescaling itself. At zero-shot time, an
unseen dataset shares the same canonical target, so the same
$(\phi_\theta,\Lambda_\theta)$ applies without adaptation. We state this formally in Proposition~\ref{prop:domain_invariant}.

Again, the variance in the inter-arrival times of the training data ensures a robust representation that precisely maps regions of Exp(1) space to input space in a manner that minimizes time RMSE. Note that it is crucial here for $\textrm{Im}(\Lambda_\theta)$ to follow a single, ideally simple distribution \textit{across datasets} to ensure a tractable objective for the decoder (Theorem~\ref{thm:time_rescaling}). Thus, a single pretrained SurF checkpoint should transfer
zero-shot. Section~\ref{exp:multihorizon} confirms it.

\section{Experiments}
\label{sec:experiments}

We assess SurF as a general-purpose generative model for irregular event streams by asking: \textbf{(i)} Can the SurF bijection reconstruct rich, non-monotonic intensity patterns? (\S\ref{exp:spikes}). \textbf{(ii)} Can a single shared $\Lambda_\theta$ match or outperform dataset-specific models on next-event prediction across six diverse corpora, and is the induced density well calibrated under the TRT? (\S\ref{exp:nextevent}). \textbf{(iii)} Does SurF remain stable under repeated autoregressive decoding, and does the canonical target allow strict zero-shot transfer to unseen datasets? (\S\ref{exp:multihorizon}). To investigate these questions, we use six benchmarks covering time scales from seconds to days: A) \textbf{Taobao}~\citep{xue2022hypro}, B) \textbf{Taxi}~\citep{whong-14-taxi}, C) \textbf{Retweet}~\citep{zhou2013learning}, D) \textbf{StackOverflow}~\citep{snapnets}, E) \textbf{Amazon}~\citep{amazon-2018}, and F) \textbf{Earthquake} (${\sim}1.3$M events, ${\sim}32$K sequences, and $75$ event types in total across the six corpora; Appendix~\ref{app:datasets}). We use RMSE for time prediction and accuracy for type prediction as our performance measures.

\subsection{Proof of Concept: Recovering Oscillatory Firing Rates}
\label{exp:spikes} Before turning to benchmarks, we verify on a controlled synthetic process what SurF promises to learn. We generated synthetic spike trains from $\lambda(t) = 0.5 + 19.5\bigl((1+\cos(\pi t))/2\bigr)^3$ (base rate $0.5$\,Hz, peak $20$\,Hz) and trained on $200$ trials of $6$\,s. SurF-MoE accurately recovers the oscillatory peaks and matches the ground-truth ISI distribution, whereas a Hawkes baseline produces a noisy intensity estimate and a heavily biased inter-spike interval (ISI) tail; SurF thus captures event-rate structures invisible to classical parametric models (Figure~\ref{fig:spikes}).

\subsection{Next-Event Prediction}
\label{exp:nextevent}
With the parameterization validated on synthetic data, we now ask whether a single jointly-trained $\Lambda_\theta$ can replace baselines trained on singular datasets across six real-world corpora.
Table~\ref{tab:nextevent} reports next-event RMSE and type accuracy for baselines trained directly on the six datasets and for the pretrained and finetuned models of all SurF variants; Table~\ref{tab:nexteventbaselines} reports the same two metrics for our own jointly-trained and finetuned runs of AttNHP, THP, and DTPP. 
We evaluate SurF in two regimes: \textbf{joint} (a single checkpoint trained on the union of all six training corpora, $\{A-F\}_{train}$ datasets, and evaluated per-test-set, $\{A-F\}_{test}$; the target's training split is included in the training mix) and \textbf{finetuned} (the same checkpoint finetuned per dataset, e.g., finetune on $A_{train}$ and test on $A_{test}$).

A strict leave-one-out zero-shot protocol is reported in
Table~\ref{tab:multihorizon}, e.g., we train on $\{A-E\}_{train}$ datasets and test on $F_{test}$. Baselines are trained from scratch, e.g., we train on $A_{train}$ dataset and test on $A_{test}$. Figure~\ref{fig:onestep} compares the next-event RMSE and type accuracy of the finetuned versions of the three SurF variants with the DTPP, NHP, and NJDTPP baselines.
\begin{table}[ht]
\vspace{-1 em}
\centering
\setlength{\tabcolsep}{3pt}
\renewcommand{\arraystretch}{0.95}
\caption{Next-event prediction: time RMSE / type accuracy (\%).
\textbf{Bold} = best, \underline{underline} = second-best. ``--'' = not
reported. $^{\dagger}$ DTPP conditions its mark prediction on the
ground-truth next-event time, which biases its type accuracy upward.}
\label{tab:nextevent}
\resizebox{\columnwidth}{!}{%
\footnotesize
\begin{tabular}{l|cccccc}
\toprule
\textbf{Model} & \textbf{Amazon} & \textbf{Retweet} & \textbf{Taxi} & \textbf{Earthquake} & \textbf{StackOverflow} & \textbf{Taobao} \\
\midrule
MHP~\citep{hawkes-71}     & 0.635 / 24.1 & 22.92 / 44.3 & 0.382 / 90.5 & -- / --      & 1.388 / 35.0 & 0.539 / 31.9 \\
RMTPP~\citep{du-16-recurrent}   & 0.620 / 31.9 & 22.31 / 55.9 & 0.371 / 90.5 & 2.087 / 45.1 & 1.376 / 35.0 & 0.531 / 44.2 \\
NHP~\citep{mei-17-neuralhawkes}     & 0.621 / 32.9 & 21.90 / 60.0 & 0.369 / \underline{91.5} & 2.294 / 45.2 & 1.372 / 45.0 & 0.531 / 44.2 \\
SAHP~\citep{zhang-2020-self}    & 0.619 / 32.3 & 22.40 / 58.4 & 0.372 / 90.3 & 1.832 / 44.3 & 1.375 / 43.9 & 0.532 / 45.4 \\
AttNHP~\citep{yang-2022-transformer}  & 0.621 / 34.7 & 22.19 / 59.9 & 0.370 / 91.3 & 1.853 / 45.1 & 1.374 / 45.0 & 0.531 / 46.4 \\
THP~\citep{zuo2020transformer}     & 0.621 / 33.9 & 22.01 / 58.5 & 0.370 / 91.3 & 2.357 / 45.0 & 1.374 / 45.0 & 0.531 / 46.4 \\
ODETPP~\citep{chen2020neural}  & 0.620 / 34.2 & 22.48 / 56.8 & 0.371 / 89.5 & 66.37 / 44.8 & 1.374 / 43.2 & 0.533 / 44.6 \\
FullyNN~\citep{omi2019fully} & 0.615 / --   & 21.92 / --   & 0.373 / --   & 2.330 / --   & 1.375 / --   & 0.529 / --   \\
IFTPP~\citep{shchur2019intensity}   & 0.618 / 32.5 & 22.18 / \underline{60.3} & 0.377 / 91.4 & 2.550 / 46.7 & 1.374 / 43.2 & 0.531 / 44.6 \\
DTPP~\citep{panos2024decomposable}    & 0.346 / \textbf{40.9}$^{\dagger}$ & 17.89 / 60.0$^{\dagger}$ & \textbf{0.283} / \textbf{92.9}$^{\dagger}$ & 1.372 / \textbf{47.3}$^{\dagger}$ & \underline{1.034} / \textbf{49.6}$^{\dagger}$ & 0.224 / 59.9$^{\dagger}$ \\
NJDTPP~\citep{zhang2024neural}  & -- / --      & 17.50 / \textbf{60.8} & 0.296 / 90.8 & 1.420 / \underline{47.2} & 1.112 / \underline{46.9} & \underline{0.130} / 48.6 \\
GRUwE~\citep{joshi2025still}   & 0.611 / \underline{38.8} & 22.21 / 59.2 & 0.369 / 91.4 & -- / --      & 1.369 / 44.7 & -- / --      \\
TimesFM~\citep{das2024decoder} & 0.386 / --   & 17.69 / --   & 0.316 / --   & 1.367 / --   & 1.149 / --   & 0.131 / --   \\
\midrule
\textsc{SurF-MoE} (joint) & \underline{0.343} / 34.4 & 16.04 / 59.2 & 0.491 / 89.7 & 1.541 / 47.1 & 1.255 / 44.8 & 0.231 / \underline{60.0} \\
\textsc{SurF-MoE} (fine)  & \textbf{0.336} / 34.7 & 16.23 / 59.7 & 0.311 / 89.5 & 1.258 / 47.0 & 1.077 / 45.1 & 0.189 / \textbf{60.3} \\
\textsc{SurF-CSB} (joint) & 0.410 / 30.4 & 15.89 / 56.8 & 0.456 / 90.3 & 1.294 / 47.1 & 1.104 / 44.5 & 0.369 / 51.5 \\
\textsc{SurF-CSB} (fine)  & 0.409 / 33.1 & 15.81 / 59.5 & 0.288 / 90.5 & \underline{1.256} / \underline{47.2} & 1.039 / 44.7 & 0.224 / 57.2 \\
\textsc{SurF-GLQ} (joint) & 0.477 / 33.3 & \underline{15.80} / 59.1 & 0.330 / 88.2 & 1.294 / \underline{47.2} & 1.098 / 44.4 & \textbf{0.126} / 53.4 \\
\textsc{SurF-GLQ} (fine)  & 0.409 / 34.2 & \textbf{15.78} / 59.8 & \underline{0.286} / 90.9 & \textbf{1.228} / \textbf{47.3} & \textbf{1.018} / 44.6 & \textbf{0.126} / 56.2 \\
\bottomrule
\end{tabular}
}
\end{table}
The best SurF variant achieves the best reported time RMSE on five of
six benchmarks and is within trial-level noise of the strongest baseline
on Taxi, the exception; type accuracy is competitive across the board. On
Earthquake, SurF-GLQ (finetuned) improves ${\sim}10\%$ over TimesFM and
DTPP and ${\sim}14\%$ over NJDTPP; all three finetuned variants beat
every baseline, and the jointly-trained SurF-CSB and SurF-GLQ
checkpoints already do so without any per-dataset finetuning. On
Retweet, SurF-CSB (finetuned) improves ${\sim}10\%$ over NJDTPP and
${\sim}12\%$ over DTPP; its jointly-trained counterpart again leads
all baselines. SurF-GLQ (finetuned) attains the best Taobao time RMSE,
and SurF-MoE (finetuned) the top Taobao type accuracy among TPP
methods. On Amazon and StackOverflow, SurF beats
DTPP on time RMSE, though both gaps are comparable to
$\sigma_{\mathrm{RMSE}}{\approx}0.058$ and should be read as effective
ties; on Taxi the gap is within $1.1\%$. The best SurF variant
outperforms the much larger foundation-model baseline TimesFM ($231$M
parameters) on every reported dataset.
\begin{figure}
    \centering
    \includegraphics[width=1\textwidth]{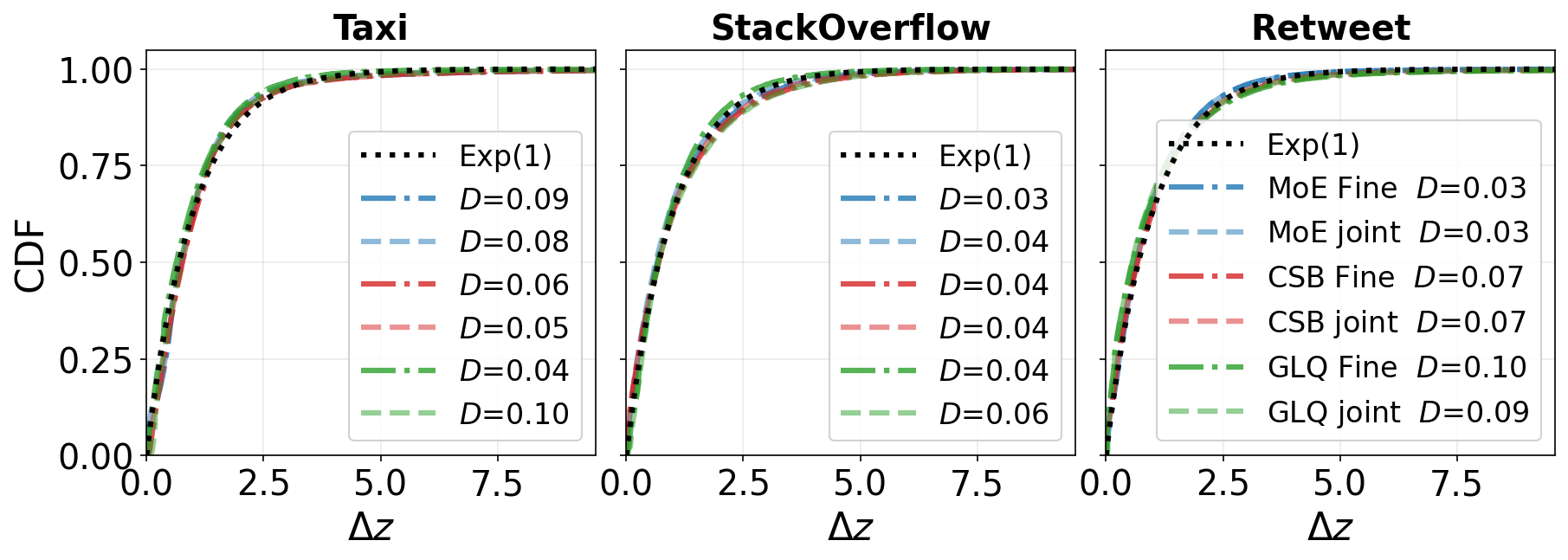}
    \caption{
        \textbf{Time-rescaling goodness-of-fit across datasets and variants.}
        Empirical CDF of the residuals
        $\Delta z_i = \Lambda(\tau_i \mid h_{i-1})$ on the test split.
        A perfectly calibrated model produces residuals that are i.i.d.\
        $\mathrm{Exp}(1)$ (black dotted curve); smaller KS statistic
        $D$ indicates better calibration. SurF achieves near-ideal
        calibration (Figure~\ref{fig:calibration_ks_cdf_full} presents the results for all six datasets).
    }
    \label{fig:calibration_ks_cdf}
\end{figure}
The jointly-trained checkpoints
serve six heterogeneous datasets from a single shared encoder,
consistent with the in-corpus scope of
Proposition~\ref{prop:domain_invariant}; out-of-corpus transfer is
evaluated separately under leave-one-out in
Table~\ref{tab:multihorizon}. Fine-tuning further improves time RMSE
on at least $5/6$ datasets per variant (CSB on $6/6$) and type
accuracy on $6/6$ for CSB and GLQ. Figure~\ref{fig:calibration_ks_cdf} shows the Kolmogorov-Smirnov (KS) statistic, denoted as D, plotted for the CDF of $\Delta z$. $D = \sup_x\left|G(x) - F(x)\right|$ measures the maximum distance between two CDFs $G$ and $F$; in this case, we let $G$ be the CDF of the true Exp(1) distribution while $F$ represents an empirical distribution from a SurF variant approximating Exp(1). We see from Figure~\ref{fig:calibration_ks_cdf} that the gains in type accuracy and time RMSE are
accompanied by well-calibrated densities: the time-rescaled residuals
$\Delta z_i = \Lambda_\theta(\tau_i \mid h_{i-1})$ should be i.i.d.\
$\mathrm{Exp}(1)$ under a perfectly specified model
(TRT~\citep{brown2002time}). All three SurF variants achieve $D \leq 0.1$
on Taxi and StackOverflow with empirical CDFs tracking the theoretical
curve almost exactly while the remaining datasets show only modest deviation.

\textbf{Goodness-of-fit via NLL.}
Beyond predictive accuracy and KS calibration, we also report held-out negative log-likelihood (NLL), the canonical TPP goodness-of-fit score, in Appendix~\ref{app:nll}; MoE achieves the lowest NLL on $4/6$ datasets and GLQ on the remaining two. The ranking among SurF variants tracks the calibration findings of Figure~\ref{fig:calibration_ks_cdf}: MoE wins on $4/6$ datasets (Taxi, StackOverflow, Earthquake, Taobao), GLQ on the two with the heaviest tails (Retweet, Amazon), and finetuning changes NLL only marginally in either direction (Appendix~\ref{app:nll}), consistent with the pretrained density already sitting near the optimum; the gains in RMSE and type accuracy are therefore accompanied by tighter density estimates.

\subsection{Multi-Horizon Forecasting}
\label{exp:multihorizon}
To test whether stability survives
iterated decoding (and whether the canonical target supports strict
zero-shot transfer to held-out datasets) we generate the next
$h \in \{1,\ldots,10\}$ events under three protocols (joint, finetuned,
leave-one-out) and measure per-horizon RMSE and type accuracy.
\begin{figure}[ht]
\vspace{-0.3em}
\centering
\includegraphics[width=\textwidth]{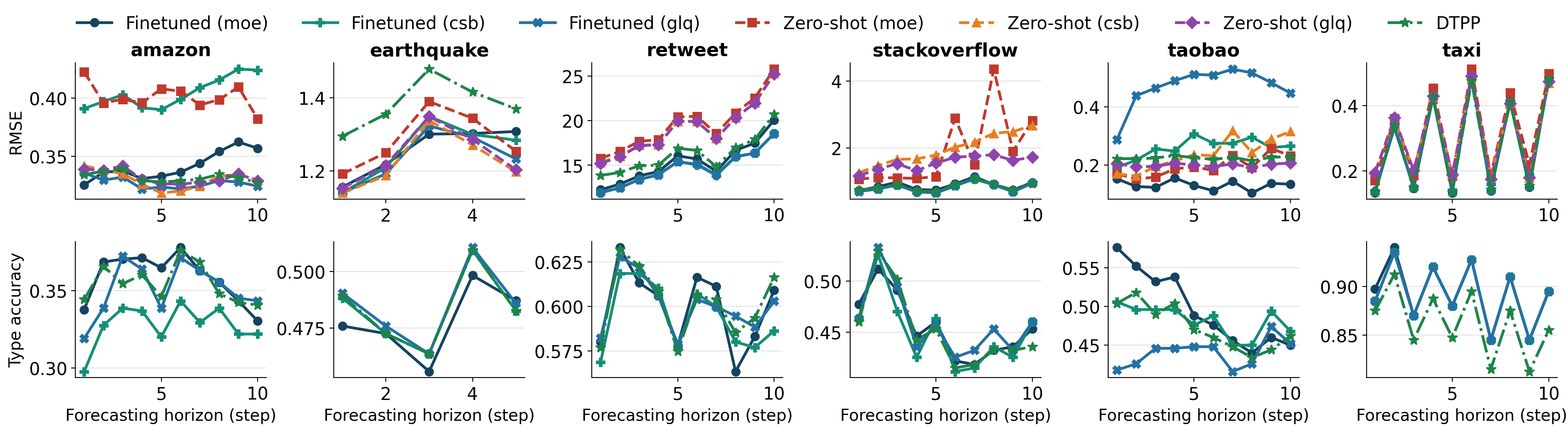}
\vspace{-1.8em}
\caption{\small Per-horizon inter-arrival RMSE (top) and type accuracy
(bottom) for finetuned (solid) and zero-shot (dashed) SurF variants
over $1$--$10$ future events (Earthquake to $5$). Zero-shot type
accuracy is omitted: the type classifier is dataset-specific and only
active after fine-tuning.}
\label{fig:multihorizon}
\vspace{-0.6em}
\end{figure}
Figure~\ref{fig:multihorizon} shows that iterated autoregressive
decoding remains stable. Finetuned RMSE is nearly flat on Amazon and
grows only modestly elsewhere; Retweet exhibits the largest compounding, but growth is monotone and
sub-linear. Type accuracy stays within a few points of the one-step value on every dataset.

Finetuned SurF matches or beats DTPP, the strongest baseline, on time RMSE on all six datasets, with margins from a tie on Taxi (where SurF additionally leads on type accuracy) to ${\sim}41\%$ on Taobao. On type accuracy, finetuned SurF leads on $4/6$ datasets among non-$^{\dagger}$ comparisons. GLQ wins outright on four datasets, ties DTPP on Taxi, and is beaten only on Taobao, where MoE wins.

\textbf{Zero-shot vs.\ baselines.} Under the strict leave-one-out
protocol, the best SurF variant beats every zero-shot baseline on $5/6$ datasets, losing only
on StackOverflow (gaps are largest where baselines collapse). Across variants, GLQ
transfers best on $3/6$, CSB on $2/6$, and MoE on Taobao---the two
more flexible variants together accounting for $5/6$ best transfers,
indicating that capacity for non-monotonic and irregular intensity
shapes generalizes even when the target is entirely unseen during
pretraining.

\textbf{Joint training.} In the joint training setting (top block of
Table~\ref{tab:multihorizon}), the MoE variant of SurF leads on time RMSE for $3/6$
datasets, CSB on $2/6$, and GLQ on $1/6$.

\begin{table}[ht]
\vspace{-1 em}
\centering
\setlength{\tabcolsep}{3pt}
\renewcommand{\arraystretch}{0.95}
\caption{Multi-horizon forecasting: inter-event time RMSE / type
accuracy (\%). \textbf{Bold} = best within block, \underline{underline}
= second-best. ``--'' = not reported. Note: $^{\dagger}$ DTPP conditions
its mark prediction on the ground-truth next-event time.}
\label{tab:multihorizon}
\resizebox{\columnwidth}{!}{%
\small
\footnotesize
\begin{tabular}{l|cccccc}
\toprule
\textbf{Model} & \textbf{Amazon} & \textbf{Retweet} & \textbf{Taxi} & \textbf{Earthquake} & \textbf{StackOverflow} & \textbf{Taobao} \\
\midrule
\textsc{AttNHP} (joint)~\citep{yang-2022-transformer}   & 1.508 / 34.3 & 19.00 / \underline{59.72} & 0.944 / 58.58 & 1.912 / 21.88 & 1.616 / 42.65 & 0.823 / 46.54 \\
\textsc{THP} (joint)~\citep{zuo2020transformer}   & 1.053 / 31.91 & 21.77 / 54.88 & 0.364 / 45.00 & 1.426 / 48.27 & 0.989 / 43.80 & 0.243 / 43.62 \\
\textsc{DTPP} (joint)~\citep{panos2024decomposable}   & \textbf{0.332 / 36.14} & 16.32 / \textbf{59.89} & \underline{0.321} / \underline{89.50} & 1.383 / \underline{48.33} & 0.861 / 44.84 & 0.224 / 47.28 \\
\textsc{SurF-MoE} (joint)     & \underline{0.336 / 35.66} & \textbf{15.26} / 59.38 & 0.323 / 89.20 & \textbf{1.349} / 48.16 & 0.866 / 45.40 & \underline{0.132} / \textbf{49.78} \\
\textsc{SurF-CSB} (joint)     & 0.402 / 31.91 & \underline{15.27} / 56.09 & \textbf{0.319} / 89.10 & \underline{1.380} / 48.29 & \textbf{0.836} / \underline{45.92} & 0.361 / 47.38 \\
\textsc{SurF-GLQ} (joint)     & 0.582 / 34.61 & 19.73 / 59.15 & 0.326 / \textbf{89.70} & 1.438 / \textbf{48.34} & \underline{0.846} / \textbf{46.06} & \textbf{0.129} / \underline{49.04} \\
\midrule
\textsc{AttNHP}~\citep{yang-2022-transformer} & 0.657 / \textbf{36.48} & 18.18 / 58.91 & 0.397 / 85.35 & 1.911 / 48.34 & 1.235 / \underline{45.89} & 0.427 / \underline{48.20} \\
\textsc{THP}~\citep{zuo2020transformer} & 0.657 / 35.53 & 18.51 / 59.65 & 0.371 / 42.90 & 2.746 / 46.84 & 1.218 / 43.00 & 0.357 / 43.62 \\
\textsc{DTPP}~\citep{panos2024decomposable}                 & \underline{0.332} / 35.45$^{\dagger}$ & 16.20 / \textbf{60.19}$^{\dagger}$ & \textbf{0.285} / 86.20$^{\dagger}$ & 1.382 / 48.35$^{\dagger}$ & 0.851 / 45.19$^{\dagger}$ & \underline{0.223} / 47.34$^{\dagger}$ \\
\textsc{AttNHP} (fine)~\citep{yang-2022-transformer} & 4.071 / 29.52 & 17.28 / 60.07 & 0.543 / 82.95 & 2.790 / \textbf{48.63} & 2.102 / 45.12 & 2.518 / 44.66 \\
\textsc{THP} (fine)~\citep{zuo2020transformer} & 0.594 / \underline{36.30} & 21.78 / 59.83 & 0.379 / 89.20 & 1.821 / 48.25 & 1.102 / 45.47 & 0.279 / 47.32 \\
\textsc{DTPP} (fine)~\citep{panos2024decomposable}                 & \underline{0.332} / 34.22$^{\dagger}$ & 16.31 / 59.72$^{\dagger}$ & 0.321 / 86.25$^{\dagger}$ & 1.383 / 48.40$^{\dagger}$ & 0.861 / 44.80$^{\dagger}$ & \underline{0.223} / 46.86$^{\dagger}$ \\
\textsc{SurF-MoE} (fine)      & 0.341 / 35.91 & 15.52 / 59.94 & 0.325 / \underline{89.30} & \underline{1.258} / 47.80 & 0.871 / 45.51 & \textbf{0.132} / \textbf{49.68} \\
\textsc{SurF-CSB} (fine)      & 0.405 / 32.75 & \underline{14.82} / 59.43 & \underline{0.320} / \textbf{89.38} & \underline{1.258} / 48.34 & \underline{0.837} / 45.75 & 0.261 / \underline{48.20} \\
\textsc{SurF-GLQ} (fine)      & \textbf{0.327} / 35.10 & \textbf{14.71} / \underline{60.10} & \textbf{0.285} / 89.10 & \textbf{1.236} / \underline{48.50} & \textbf{0.825} / \textbf{45.90} & 0.468 / 44.00 \\
\midrule
\textsc{AttNHP} (zero-shot (leave-one-out)) & 1.975 / -- & 19.79 / -- & 0.698 / -- & 1.482 / -- & \underline{1.007} / -- & 0.771 / -- \\
\textsc{THP} (zero-shot (leave-one-out)) & 9.127 / -- & \underline{19.23} / -- & 2.759 / -- & 1.447 / -- & 3.676 / -- & 3.922 / -- \\
\textsc{DTPP} (zero-shot (leave-one-out)) & 11.35 / -- & 19.90 / -- & 0.440 / -- & 1.338 / -- & \textbf{0.861} / -- & 0.795 / -- \\
\textsc{SurF-MoE} (zero-shot (leave-one-out)) & 0.401 / -- & 19.86 / -- & 0.349 / -- & 1.287 / -- & 2.172 / -- & \textbf{0.198} / -- \\
\textsc{SurF-CSB} (zero-shot (leave-one-out)) & \textbf{0.330} / -- & 19.35 / -- & \underline{0.333} / -- & \textbf{1.229} / -- & 2.016 / -- & 0.243 / -- \\
\textsc{SurF-GLQ} (zero-shot (leave-one-out)) & \underline{0.332} / -- & \textbf{19.10} / -- & \textbf{0.310} / -- & \underline{1.241} / -- & 1.559 / -- & \underline{0.200} / -- \\
\bottomrule
\end{tabular}
}
\vspace{-1. em}
\end{table}
\subsection{Ablations}
\label{exp:ablations} Here, we ablate three design choices of SurF: (i) the
\emph{conformer} block inserted before multi-head attention; (ii) the
\emph{loss schedule}
(\emph{Equal}/\emph{Staged}/\emph{Hybrid}); and (iii) the
\emph{quadrature configuration} of SurF-GLQ---Gauss--Legendre points
$Q$ and log-intensity references $J$. Full tables and per-bin
breakdowns are in Appendix~\ref{app:hp_sensitivity}. The Hybrid schedule used by full
SurF sits near the Pareto frontier across all six benchmarks: never
worst on time RMSE, best on Retweet time RMSE and tied-best on Taobao
type accuracy, and within noise of the column-best elsewhere. The Staged
schedule collapses on Retweet (RMSE jumps from ${\sim}16$ to $22.3$)
because dropping the SurF likelihood during finetuning discards the
gradient signal that anchors long-tail inter-arrivals. The conformer
block helps under the Staged schedule---most on Earthquake and
Taxi---and is harmful under the Equal schedule. Across a $100$-trial random
hyperparameter sweep, increasing $Q{=}4{\to}8$ reduces median time
RMSE by $\sim$$0.08$ and the trial-to-trial standard deviation by ${\sim}3\times$
($\sigma_{\text{RMSE}}$: $0.170{\to}0.058$); going further to $Q{=}12$
yields no measurable gain (median difference $0.04$, well within
$\sigma{\approx}0.05$, and best-trial accuracy actually drops
$49.94{\to}49.49\%$). This validates the $Q{=}8$ default of
Section~\ref{sec:training}: per-interval GLQ has already converged at $Q{=}8$. A similar saturation supports the $J{=}4$ default. Across all $100$ successful trials, time RMSE varies in $[8.725, 9.242]$ with $79\%$ within $1\%$ of the best run;
the top-$10$ configurations span a window of $0.011$ RMSE units (an
order of magnitude smaller than the gap to the closest baseline in
Table~\ref{tab:nextevent}) across a mix of $Q$, $J$, conformer, and
batch-size settings, indicating several near-optimal HP regions rather
than a single fragile peak.

\textbf{Computational Efficiency}
Based on Table~\ref{tab:cost_amazon_so} and against NJDTPP, the most directly comparable quadrature-based method, SurF-MoE trains $10.9\times$ faster on Amazon and $12.4\times$ on StackOverflow, with inference gaps above $50\times$; it is also ${\sim}5.8\times$ faster than TimesFM at inference. Against DTPP, SurF-MoE trains ${\sim}1.4\times$ faster on both datasets at ${\sim}36\%$ lower peak memory. 
\begin{table}[ht]

\centering\small
\caption{Training and inference cost on Amazon and StackOverflow (H100, batch $512$, inference averaged over $20$ sequences). Train: ms per batch. Infer.: ms per sequence. ``--'' = not run.}
\label{tab:cost_amazon_so}
\begin{tabular}{l|ccc|ccc}
\toprule
& \multicolumn{3}{c|}{\textbf{Amazon}} & \multicolumn{3}{c}{\textbf{StackOverflow}} \\
\textbf{Model} & \textbf{Train} & \textbf{Infer.} & \textbf{Mem (GB)}
               & \textbf{Train} & \textbf{Infer.} & \textbf{Mem (GB)} \\
\midrule
AttNHP  & 90.9 & 9.91 & 13.1  & 81.3 & 10.3 & 11.3 \\
IFTPP   & 5.96 & 2.73 & 0.26  & --   & --   & --   \\
ODETPP  & 240  & 102  & 0.37  & 248  & 125  & 0.27 \\
THP     & 6.61 & 4.77 & 0.81  & 7.51 & 5.92 & 0.59 \\
DTPP    & 104  & 5.57 & 8.93  & 115  & 7.42 & 9.84 \\
NJDTPP  & 797  & 373  & 0.544 & 1000 & 387  & 0.77 \\
TimesFM & --   & 39.7 & --    & --   & 40.4 & --   \\
\midrule
\textbf{SurF-MoE} & 72.9 & 6.82 & 5.71 & 80.8 & 6.80 & 6.23 \\
\textbf{SurF-CSB} & 73.3 & 6.79 & 5.71 & 87.4 & 7.21 & 6.71 \\
\textbf{SurF-GLQ} & 297 & 13.5 & 19.4 & 314 & 13.7 & 21.1 \\
\bottomrule
\end{tabular}
\vspace{-1.5 em}
\end{table}
\section{Related Work}
\textbf{Generative models for time series.}
Foundation models have advanced rapidly in
language~\citep{brown2020language,touvron2023llama},
vision~\citep{radford2021learning,oquab2023dinov2}, and regularly
sampled time series (TimesFM~\citep{das2024decoder}, Alternators ~\citep{rezaei2025alternators},
Chronos~\citep{ansari2024chronos}, Lag-Llama~\citep{rasul2024lag},
Moirai~\citep{woo2024unified}). All assume fixed sampling rates and
synchronized channels, an assumption that fails for asynchronous event
streams where inter-arrival times span orders of magnitude and timing
itself carries meaning. SurF targets this regime directly. \textbf{Neural temporal point processes.}
Neural TPPs~\citep{shchur2021neural,rezaei2021real} offer flexible
event modeling but have not reached foundation-model scale, ours
included. RNN-based methods (RMTPP~\citep{du-16-recurrent,rezaei2018comparison},
Neural Hawkes~\citep{mei-17-neuralhawkes}) suffer from gradient pathologies
and sequential bottlenecks; Transformer variants
(SAHP~\citep{zhang-2020-self}, THP~\citep{zuo2020transformer}, and
extensions~\citep{li2023sparse,gu2021attentive,xue2022hypro})
parallelize training but still require numerical quadrature of
$\int_0^T \lambda^*(s\mid\mathcal{H}_s)\,ds$ inside the likelihood, at
a cost that grows with the observation window. Intensity-free methods
(IFTPP~\citep{shchur2019intensity}) avoid the integral but discard the
intensity representation that makes the TRT usable. \textbf{Cumulative-hazard neural TPPs.}
Closest to SurF are FullyNN~\citep{omi2019fully} and the log-normal
mixture of~\citet{shchur2019intensity}, which parameterize
$\Lambda_\theta$ directly and recover $\lambda_\theta$ by
differentiation. We differ on three fronts.
\emph{Conceptually,} prior work treats $\Lambda_\theta$ as a
likelihood-tractability trick, whereas SurF reinterprets the TRT
itself as a learnable bijection to a canonical exponential space,
yielding the domain (dataset)-invariance result in
Proposition~\ref{prop:domain_invariant}.
\emph{Architecturally,} prior work uses single-dataset feedforward or
recurrent encoders; SurF pairs the cumulative-intensity head with a
Transformer encoder and multi-dataset pretraining, and offers three
monotone parameterizations (MoE, CSB, GLQ) with distinct
expressiveness--cost trade-offs
(Appendix~\ref{app:variants_quadrature}).
\emph{Empirically,} SurF is the first cumulative-hazard TPP to report
cross-dataset transfer. The jointly-trained checkpoint outperforms
per-dataset FullyNN on all six benchmarks for the best variant per
dataset, and under the strict leave-one-out protocol of
Table~\ref{tab:multihorizon} the held-out checkpoint ( which never
sees the target corpus) still beats every classical and
neural-autoregressive baseline on $5/6$ datasets. See
Appendix~\ref{app:related_cumhazard} for a full comparison. \textbf{Relation to intensity-free and flow-based models.}
Some intensity-free and flow-based TPPs also use invertible transformations from a simple base variable to event times; the difference is what the base variable is. IFTPP~\citep{shchur2019intensity} fits the conditional density of the inter-event time with mixture parameters emitted by the history encoder, so its latents are dataset-specific and nothing in the model or loss pushes two corpora into a common space. FullyNN~\citep{omi2019fully} does have the compensator and does use the $\mathrm{Exp}(1)$ property, but only to solve $\Lambda^{\text{FNN}}_\theta(\tau\mid\mathbf{h}_i)=\log 2$ by bisection for a median point prediction, and is fitted one dataset at a time. In SurF the target is not chosen: the TRT fixes it before training, identically for every corpus. To our knowledge no prior neural TPP uses that shared target to train one set of weights across multiple corpora, or reports cross-dataset transfer of the resulting checkpoint (Table~\ref{tab:loo_degradation}).

\section{Conclusion}
We introduced \textbf{SurF}, a generative model for irregularly sampled
multivariate event streams that reinterprets the Time Rescaling Theorem
as a learnable, invertible mapping between event sequences and
canonical $\mathrm{Exp}(1)$ noise. SurF builds on the cumulative-hazard
lineage of neural TPPs and reframes the TRT as a normalizing flow. Our
contributions are an explicit bidirectional formulation grounded in a
bijectivity statement that follows from the inverse function theorem
(Theorem~\ref{thm:reverse_rescaling}), three monotone parameterizations
with different expressiveness--cost trade-offs
(Appendix~\ref{app:variants_quadrature}), and a Transformer-amortized
encoder that makes multi-dataset pretraining practical. Under a strict
leave-one-out protocol, a held-out SurF checkpoint beats every
classical and neural-autoregressive baseline on $5/6$ benchmarks and
beats every baseline on Amazon and
Earthquake; empirical evidence of partial cross-dataset transfer for a
learned cumulative intensity.

\textbf{Limitations and future work.}
The bijectivity of Theorem~\ref{thm:reverse_rescaling} requires a
strictly positive intensity, enforced by a small floor
$\lambda_{\text{floor}}$ that introduces a bounded compensator bias
(Appendix~\ref{app:floor}); negligible on our benchmarks but
potentially material for processes with genuine dead-time (refractory
neural firing, market-closure windows). The MoE parameterization
admits closed-form integration but is restricted to monotonically
decaying between-event intensity; CSB and GLQ lift this at modest
cost, and richer tractable parameterizations remain a natural next
step. Finally, our six-dataset pre-training
corpus is sufficient to demonstrate cross-dataset transfer but remains
orders of magnitude smaller than language and vision corpora; scaling
is the key milestone for foundation-scale models of asynchronous event
streams.

\bibliographystyle{unsrtnat}
\bibliography{main}
\newpage
\appendix
\newpage
\tableofcontents
\newpage
\setcounter{theorem}{0}

\section{Time Rescaling Theorem and Bijectivity Proofs}
\label{app:trt-proof}

Consider a temporal point process defined on $\mathbb{R}^+$ generating event
times $\{t_1, t_2, \ldots, t_N\}$ with a conditional intensity function
$\lambda^*(t\mid\mathcal{H}_t) > 0$ for all $t \geq 0$, where
$\mathcal{H}_t$ represents the history of events up to time $t$. The
intensity function characterizes the instantaneous rate of event occurrence
and fully describes the probabilistic structure of the process.

\begin{theorem}[Time Rescaling Theorem, restated]\label{thm:time_rescaling_app}
Let $\{t_i\}_{i=1}^N$ be a realization of a point process with conditional
intensity function $\lambda^*(t\mid\mathcal{H}_t) > 0$. Define the
cumulative intensity function
$\Lambda^*(t) = \int_0^t \lambda^*(s\mid\mathcal{H}_s)\,ds$, assumed to
satisfy $\Lambda^*(t)\to\infty$ as $t\to\infty$, and the
rescaled event times $z_i = \Lambda^*(t_i)$. Then
$\{\Delta z_i = z_i - z_{i-1}\}_{i=1}^N$ with $z_0 = 0$ are i.i.d.\
exponential random variables with unit rate.
\end{theorem}

\begin{proof}
Fix $i \in \{1,\ldots,N\}$ and condition on $\mathcal{H}_{t_i}$, the
first $i-1$ events, with $t_0 = 0$. By definition of the conditional
intensity, the conditional survival function of the $i$-th event time is,
for $t \geq t_{i-1}$,
\begin{equation}
\Pr(t_i > t \mid \mathcal{H}_{t_i})
= \exp\!\Bigl(-\int_{t_{i-1}}^{t} \lambda^*(s\mid\mathcal{H}_s)\,ds\Bigr)
= \exp\!\bigl(-[\Lambda^*(t) - \Lambda^*(t_{i-1})]\bigr).
\end{equation}
Since $\lambda^* > 0$, $\Lambda^*$ is strictly increasing and hence
invertible on its range, and $\Lambda^*(t)\to\infty$ drives the
right-hand side to $0$, so $t_i < \infty$ almost surely. For $u \geq 0$
set $t = (\Lambda^*)^{-1}(\Lambda^*(t_{i-1}) + u)$; monotonicity gives
$\{\Delta z_i > u\} = \{t_i > t\}$, so
\begin{equation}
\Pr(\Delta z_i > u \mid \mathcal{H}_{t_i}) = \exp(-u),
\qquad u \geq 0 .
\end{equation}
The conditional law of $\Delta z_i$ is thus $\mathrm{Exp}(1)$ whatever
$\mathcal{H}_{t_i}$ is, so $\Delta z_i$ is independent of
$\mathcal{H}_{t_i}$ and in particular of $\Delta z_1,\ldots,\Delta z_{i-1}$,
which are $\mathcal{H}_{t_i}$-measurable. Applying this for
$i = 1,\ldots,N$ gives that the $\Delta z_i$ are i.i.d.\ $\mathrm{Exp}(1)$.
\end{proof}

The crucial aspect of this transformation is its bijectivity. The
condition $\lambda^*(t\mid\mathcal{H}_t) > 0$ ensures that $\Lambda^*(t)$
is strictly increasing.\\

\begin{corollary}\label{cor:time_rescaling_corollary}
Let $\{t_i\}_{i=1}^N$ be a realization of a point process with conditional
intensity function $\lambda_0(t\mid\mathcal{H}_t) > 0$. For any $t > 0$, let $t' =  \max\left(\{t_j : (t_j,k_j)\in\mathcal{H}_t\}\cup\{0\}\right)$, and write $\tilde\lambda(\Delta t\mid \mathcal{H}_t) = \lambda_0(t\mid \mathcal{H}_t)$, where $\Delta t = t - t'$. Define the
cumulative intensity functions
$\Lambda_0(t) = \int_0^t \lambda_0(s\mid\mathcal{H}_s)\,ds$, assumed to satisfy $\Lambda_0(t)\to\infty$ as $t\to\infty$, and $\tilde\Lambda(\Delta t \mid \mathcal{H}_t) = \int_0^{\Delta t} \tilde\lambda(s\mid\mathcal{H}_t)\,ds$. Then $\{\tilde\Lambda(\tau_i \mid \mathcal{H}_{t_i})\}_{i=1}^N$ are i.i.d. exponential random variables with unit rate, where $\tau_i = t_i - t_{i-1}$ and $t_0 = 0$. Note that under the convention $\mathcal{H}_t=\{(t_j,k_j):t_j<t\}$ we have $\mathcal{H}_{t_i}=\{(t_j,k_j)\}_{j=1}^{i-1}$, the history available for predicting the $i$-th event, which is what the encoding $\mathbf{h}_{i-1}$ represents.
\end{corollary}

\begin{proof}
    We know from Theorem~\ref{thm:time_rescaling_app} that $\Lambda_0(t_i) - \Lambda_0(t_{i-1})$ are i.i.d.\ Exp(1) random variables. Observe that:
    \begin{align}
        \Lambda_0(t_i) - \Lambda_0(t_{i-1}) &= \int_{t_{i-1}}^{t_i}\lambda_0(s\mid \mathcal{H}_s)ds \\
        &= \int_0^{\tau_i}\lambda_0(s+t_{i-1} \mid \mathcal{H}_{s + t_{i-1}})ds\\
        &= \int_0^{\tau_i} \tilde\lambda(s\mid \mathcal{H}_{t_i})ds\\
        &= \tilde\Lambda(\tau_i\mid \mathcal{H}_{t_i})
    \end{align}
\end{proof}

\begin{lemma}[Invertibility and Regularity]\label{lem:invertibility}
Assume $\lambda^*(\cdot\mid\mathcal{H}_\cdot)$ is continuous and
$\lambda^*(t\mid\mathcal{H}_t)\geq\lambda_{\min}>0$ for all
$t\in\mathbb{R}^+$. Then $\Lambda^*$ is a $C^1$ bijection from
$\mathbb{R}^+$ to $\mathbb{R}^+$ with $C^1$ inverse $(\Lambda^*)^{-1}$.
If, in addition, $\lambda^*$ is $C^k$ for some $k\geq 0$, then $\Lambda^*$
and $(\Lambda^*)^{-1}$ are $C^{k+1}$.
\end{lemma}

\begin{proof}
By continuity of $\lambda^*$ and the fundamental theorem of calculus,
$\Lambda^*$ is $C^1$ on $\mathbb{R}^+$ with
$(\Lambda^*)'(t) = \lambda^*(t\mid\mathcal{H}_t)$. The lower bound
$(\Lambda^*)'(t)\geq\lambda_{\min}>0$ implies $\Lambda^*$ is strictly
increasing (hence injective) and
$\Lambda^*(t)\geq\lambda_{\min}\,t\to\infty$ as $t\to\infty$. Together with
$\Lambda^*(0)=0$ and continuity, this gives surjectivity onto
$\mathbb{R}^+$, so $\Lambda^*$ is a $C^1$ bijection. The Inverse Function
Theorem applied at every point (where $(\Lambda^*)'\neq 0$) yields a $C^1$
inverse. If $\lambda^*$ is $C^k$, then $\Lambda^*$ is $C^{k+1}$ by
repeated differentiation, and $(\Lambda^*)^{-1}$ inherits this regularity
via IFT.
\end{proof}

\begin{proposition}[Reverse-direction realization]\label{prop:reverse_direction}
Under the hypotheses of Lemma~\ref{lem:invertibility}, let $\{z_i\}$ be a
realization of a unit-rate Poisson process on $\mathbb{R}^+$. Then
$\{t_i=(\Lambda^*)^{-1}(z_i)\}$ is a realization of the original point
process with conditional intensity $\lambda^*$.
\end{proposition}

\begin{proof}
Let $\mathcal{L}_\lambda$ denote the law of the original point process and
$\mathcal{L}_1$ the law of the unit-rate Poisson process. Theorem~%
\ref{thm:time_rescaling_app} states that the map
$\Phi:\{t_i\}\mapsto\{\Lambda^*(t_i)\}$ pushes $\mathcal{L}_\lambda$ forward
to $\mathcal{L}_1$, i.e.\ $\Phi_\ast\mathcal{L}_\lambda=\mathcal{L}_1$.
Lemma~\ref{lem:invertibility} ensures $\Phi$ is a measurable bijection on
event sequences, so the inverse pushforward gives
$(\Phi^{-1})_\ast\mathcal{L}_1=\mathcal{L}_\lambda$. Hence
$\{(\Lambda^*)^{-1}(z_i)\}\sim\mathcal{L}_\lambda$.
\end{proof}

\begin{remark}
The lower bound $\lambda^*\geq\lambda_{\min}>0$ --- rather than mere
positivity $\lambda^*>0$ --- is essential. Without it, $\Lambda^*$ can have
a finite limit and fail to be surjective onto $\mathbb{R}^+$: e.g.,
$\lambda^*(t) = 1/(1+t)^2$ is strictly positive but
$\Lambda^*(t) = 1 - 1/(1+t) \to 1$, in which case $\Lambda^*$ is a $C^1$
bijection only onto $[0,1)$. In SurF, the floor $\lambda_{\text{floor}}$
enforces exactly the uniform lower bound required here
(Appendix~\ref{app:floor}).
\end{remark}

This bijectivity is fundamental to SurF's amortized loss: it guarantees
that the learned mapping $\Lambda_\theta$ between temporal and
exponential spaces is perfectly reversible, enabling exact likelihood
computation via change of variables and lossless generative sampling via
inversion.

\subsection{Near-Zero Intensity Regimes and Likelihood Consistency}\label{app:floor}

The bijectivity in Theorem~\ref{thm:reverse_rescaling} requires
$\lambda_{\min}>0$. Real processes (neural refractoriness, market-closure
windows, after-hours IoT) can exhibit regimes where the true intensity
is effectively zero. We handle this structurally rather than with ad hoc
clamping.

\paragraph{Structural positivity.} All three variants are structurally
lower-bounded. SurF-MoE uses
$\gamma_j = \mathrm{softplus}(\hat\gamma_j)+\varepsilon$ with
$\varepsilon=10^{-4}$; SurF-CSB and SurF-GLQ apply analogous
$\mathrm{softplus}+\varepsilon$ to their positive parameters. Every
variant additionally carries $\lambda_{\text{floor}}\Delta t$ in
$\Lambda_\theta$ (\eqref{eq:amort_intensity}), giving a hard floor
$\lambda_{\text{floor}}>0$ that ensures the conditions of
Theorem~\ref{thm:reverse_rescaling}.

\paragraph{Bias from the floor.} Let
$\lambda_\theta^\dagger = \lambda_\theta + \lambda_{\text{floor}}$. The floor raises the compensator by $\tau_i\lambda_{\text{floor}}$ per interval and simultaneously raises every log-intensity term, so the net likelihood bias is
\[
b(\lambda_{\text{floor}}) = T\lambda_{\text{floor}}
- \sum_{i=1}^{N}\log\!\Bigl(1+\tfrac{\lambda_{\text{floor}}}{\lambda_\theta(\tau_i\mid\mathbf{h}_{i-1})}\Bigr)
\;\le\; T\lambda_{\text{floor}},
\] strictly below the stated bound whenever the intensity is finite, and vanishing as $\lambda_{\text{floor}}\to 0$, so the estimator is consistent in the limit. For $\lambda_{\text{floor}} = 10^{-4}$ and typical normalized $T \leq 10^3$, the upper bound
is at most $0.1$ nats per sequence, below the NLL noise floor on all six benchmarks.
\paragraph{Piecewise-zero regimes.} For processes with genuine dead-time,
we recommend one of two strategies, neither of which affects the core
theory.
\begin{itemize}[leftmargin=*]
\item \textbf{Learnable floor:} treat $\lambda_{\text{floor}}$ as a
parameter bounded below by $10^{-6}$, letting the model absorb mass of
near-zero regions without violating bijectivity.
\item \textbf{Explicit masking:} for known refractory windows
$[t_i, t_i + \delta]$, set
$\lambda_\theta \equiv \lambda_{\text{floor}}$ inside the window and
integrate only outside. The Jacobian factorization is unchanged because
masked regions contribute a constant to the compensator.
\end{itemize}
We use the learnable-floor strategy throughout and observe no practical
difference from explicit masking on refractoriness-relevant datasets
(Retweet, Earthquake).

\paragraph{Sampling.} Newton's method
(Algorithm~\ref{alg:surf_sampling_newton}) is unaffected: positivity of
$\lambda_\theta$ is exactly what guarantees global convergence, and the
floor improves numerical conditioning of the update $(F-z)/F'$ at large
$\Delta t$.
\paragraph{Learnable floor in practice.} When $\lambda_{\text{floor}}$
is treated as a trainable parameter (bounded below by $10^{-6}$), the
optimizer converges to values in
$[2\!\times\!10^{-5},\,8\!\times\!10^{-5}]$ across all six datasets ---
comfortably inside the flat plateau --- and yields metrics indistinguishable from the fixed $10^{-4}$ default.
This matches the recommendation above and shows the estimator is
insensitive to the exact choice within the plateau.
\subsection{Derivation of the Survival Function \texorpdfstring{$S(t)$}{S(t)}}

The conditional intensity function can be defined in terms of the density
$f_T(t)$ and survival function $S(t) = \Pr(T > t\mid\mathcal{H}_t)$ as
\begin{equation}
\lambda^*(t\mid\mathcal{H}_t) = \frac{f_T(t\mid\mathcal{H}_t)}{S(t\mid\mathcal{H}_t)}.
\end{equation}
Since $f_T(t\mid\mathcal{H}_t) = -\frac{d}{dt}S(t\mid\mathcal{H}_t)$,
\begin{equation}
\lambda^*(t\mid\mathcal{H}_t)
= \frac{-\frac{d}{dt}S(t\mid\mathcal{H}_t)}{S(t\mid\mathcal{H}_t)}
= -\frac{d}{dt}\ln S(t\mid\mathcal{H}_t).
\end{equation}
Integrating from $0$ to $t$ and using $S(0\mid\mathcal{H}_0) = 1$,
\begin{equation}
\Lambda^*(t) = \int_0^t \lambda^*(s\mid\mathcal{H}_s)\,ds
= -\ln S(t\mid\mathcal{H}_t),
\end{equation}
so $S(t\mid\mathcal{H}_t) = \exp(-\Lambda^*(t))$.

\section{Likelihood Derivation via Change of Variables}\label{app:likelihood-derivation}

We derive the likelihood for inter-arrival times
$\boldsymbol{\tau} = (\tau_1, \ldots, \tau_N)$ through the TRT
transformation to exponential space
$\boldsymbol{\Delta z} = (\Delta z_1, \ldots, \Delta z_N)$:
\begin{equation}
\Delta z_i
= \Lambda_\theta(t_i) - \Lambda_\theta(t_{i-1})
= \int_{t_{i-1}}^{t_i} \lambda^{\dagger}_\theta(s\mid\mathcal{H}_s)\,ds,
\end{equation}
where $t_i = \sum_{j=1}^i \tau_j$ are cumulative event times.

\subsection{Jacobian Analysis}

The Jacobian $J = \partial\boldsymbol{\Delta z}/\partial\boldsymbol{\tau}$
is lower triangular in general, but its determinant depends only on the
diagonal entries.

\paragraph{General case (lower triangular).} Since
$t_i = \sum_{j=1}^i \tau_j$, both limits of
$\Delta z_i = \int_{t_{i-1}}^{t_i} \lambda^{\dagger}_\theta\,ds$ depend on
$\tau_1, \ldots, \tau_i$, and the history $\mathcal{H}_s$ for
$s \in (t_{i-1}, t_i]$ may depend on all earlier event times through the
encoding. Hence
\begin{equation}
\frac{\partial \Delta z_i}{\partial \tau_j} =
\begin{cases}
\neq 0 & j \leq i, \\
0 & j > i,
\end{cases}
\end{equation}
making $J$ lower triangular. By Leibniz's rule applied to the upper
limit of integration,
\begin{equation}\label{eq:jac_diag}
\frac{\partial \Delta z_i}{\partial \tau_i}
= \lambda^{\dagger}_\theta(t_i\mid\mathcal{H}_{t_i}),
\end{equation}
regardless of how the history encoding is computed. The determinant of a
lower-triangular matrix is the product of its diagonal entries, so
\begin{equation}\label{eq:jac_det}
\det J = \prod_{i=1}^N \lambda^{\dagger}_\theta(t_i\mid\mathcal{H}_{t_i}).
\end{equation}
Since $\lambda^{\dagger}_\theta \geq \lambda_{\text{floor}} > 0$ everywhere, $\det J > 0$ and the
transformation is locally invertible.

\paragraph{Simplification under our architecture (diagonal).} In SurF,
the history encoding $\mathbf{h}_{i-1}$ is computed once at event time
$t_{i-1}$ and held constant throughout the interval $(t_{i-1}, t_i]$. For
$s \in (t_{i-1}, t_i]$, the intensity depends only on elapsed time
$u = s - t_{i-1}$ and the frozen encoding $\mathbf{h}_{i-1}$, so
\begin{equation}
\Delta z_i
= \int_0^{\tau_i} \lambda^{\dagger}_\theta(u\mid\mathbf{h}_{i-1})\,du
= \Lambda_\theta(\tau_i \mid \mathbf{h}_{i-1}).
\end{equation}
This holds regardless of whether $\Lambda_\theta$ is computed in closed
form (MoE, CSB) or by quadrature (GLQ). For $j < i$: $\tau_i$ is
independent of $\tau_j$, and $\mathbf{h}_{i-1}$ is fixed with respect to
$\tau_j$ (it is recomputed only at event boundaries), giving
$\partial \Delta z_i / \partial \tau_j = 0$. For $j > i$: neither
$\tau_i$ nor $\mathbf{h}_{i-1}$ is affected, giving
$\partial \Delta z_i / \partial \tau_j = 0$. Therefore
\begin{equation}
J = \operatorname{diag}\!\bigl(\lambda^{\dagger}_\theta(t_1\mid\mathcal{H}_{t_1}),
\ldots, \lambda^{\dagger}_\theta(t_N\mid\mathcal{H}_{t_N})\bigr),
\end{equation}
and the determinant remains
$\det J = \prod_{i=1}^N \lambda^{\dagger}_\theta(t_i\mid\mathcal{H}_{t_i})$.

\begin{remark}
The diagonal structure is a convenience of our architectural choice, not
a requirement for correctness. Any encoder producing $\mathbf{h}_{i-1}$
with continuous dependence on all past event positions would yield a
lower-triangular (but not diagonal) Jacobian. Because the determinant of
a triangular matrix still equals the product of its diagonal entries,
all subsequent derivations --- the change-of-variables likelihood below,
the SurF loss, and the gradient bounds of
Appendix~\ref{app:surf-grad} --- remain valid without modification.
\end{remark}

\subsection{Change of Variables Formula}

For the transformation $\boldsymbol{\tau} \mapsto \boldsymbol{\Delta z}$,
\begin{align}
p_\theta(\boldsymbol{\tau})
&= p_{\boldsymbol{\Delta z}}(\boldsymbol{\Delta z}(\boldsymbol{\tau}))\cdot |\det J|, \\
\log p_\theta(\boldsymbol{\tau})
&= \log p_{\boldsymbol{\Delta z}}(\boldsymbol{\Delta z}) + \log|\det J|.
\end{align}
Since $\Delta z_i \sim \mathrm{Exp}(1)$ independently,
$\log p_{\boldsymbol{\Delta z}}(\boldsymbol{\Delta z}) = -\sum_i \Delta z_i$,
so
\begin{equation}
\log p_\theta(\boldsymbol{\tau})
= -\sum_{i=1}^N \Lambda_\theta(\tau_i \mid \mathbf{h}_{i-1})
+ \sum_{i=1}^N \log \lambda^{\dagger}_\theta(\tau_i \mid \mathbf{h}_{i-1}).
\end{equation}

\subsection{Observation Window Correction}

For data observed in $[0, T]$, we include the survival probability of
the final gap $[t_N, T]$:
\begin{equation}
\Pr(\text{no event in }(t_N, T]) = \exp\!\bigl(-\Lambda_\theta(\Delta T \mid \mathbf{h}_N)\bigr).
\end{equation}
The complete negative log-likelihood is
\begin{equation}
\boxed{\mathcal{L}_{\text{SurF}}(\theta)
= \sum_{i=1}^N \Lambda_\theta(\tau_i \mid \mathbf{h}_{i-1})
+ \Lambda_\theta(\Delta T \mid \mathbf{h}_N)
- \sum_{i=1}^N \log \lambda^{\dagger}_\theta(\tau_i \mid \mathbf{h}_{i-1}).}
\end{equation}
This is the unified amortized loss. The first two terms are forward
passes of $\Lambda_\theta$; the third is computed via automatic
differentiation (or closed form for MoE).
\begin{figure}[ht]
\centering
\includegraphics[width=1\textwidth]{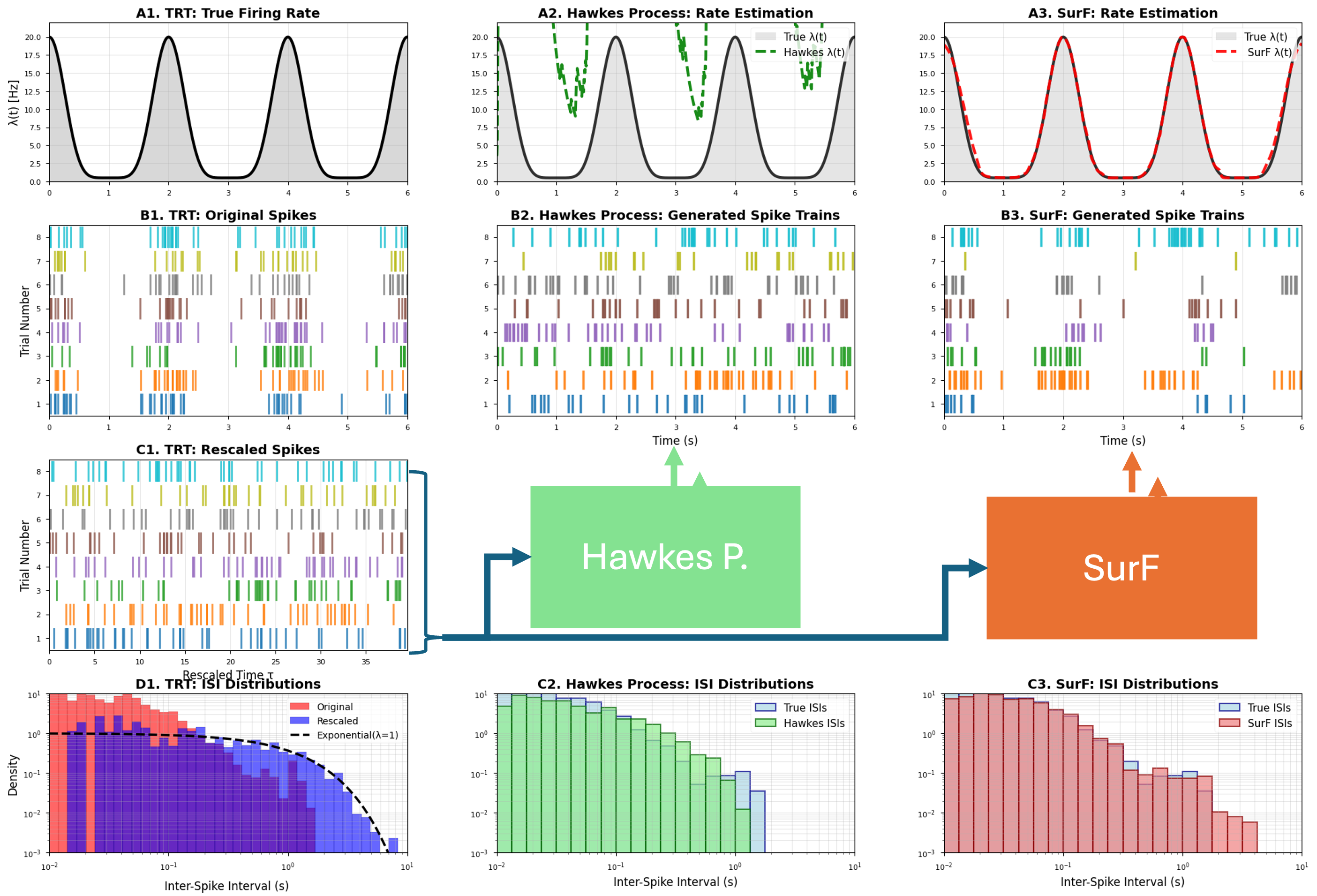}
\caption{\small SurF recovers the true oscillatory intensity (A3) and ISI
distribution (C3); Hawkes fails on both (A2, C2).}
\label{fig:spikes}
\end{figure}
\section{Held-Out Negative Log-Likelihood}
\label{app:nll}
Table~\ref{tab:nll} reports per-event NLL of the joint inter-arrival/type density $p_\theta(\tau, k \mid \mathbf{h}_{i-1})$ for all three variants, pretrained (combined-corpus checkpoint, evaluated zero-shot) and finetuned; Table~\ref{tab:loo_nll} gives a controlled leave-one-out comparison against the strongest baseline.

\begin{table}[ht]
\centering
\setlength{\tabcolsep}{4pt}
\renewcommand{\arraystretch}{0.95}
\caption{Held-out per-event NLL of SurF variants (lower is better). \textbf{Bold} = best within row across all six configurations. ``Pre'' denotes the unified pretrained checkpoint evaluated zero-shot; ``Fine'' denotes the same checkpoint adapted per dataset.}
\label{tab:nll}
\footnotesize
\begin{tabular}{l|cc|cc|cc}
\toprule
& \multicolumn{2}{c|}{\textsc{MoE}} & \multicolumn{2}{c|}{\textsc{GLQ}} & \multicolumn{2}{c}{\textsc{CSB}} \\
\textbf{Dataset} & \textbf{Pre} & \textbf{Fine} & \textbf{Pre} & \textbf{Fine} & \textbf{Pre} & \textbf{Fine} \\
\midrule
Taxi          & $-0.9341$           & $\mathbf{-0.9503}$  & $-0.6914$           & $-0.7670$           & $-0.6407$ & $-0.6835$ \\
StackOverflow & $\mathbf{+0.4687}$  & $+0.4834$           & $+0.6423$           & $+0.5833$           & $+0.6969$ & $+0.7059$ \\
Retweet       & $+2.5651$           & $+2.5557$           & $+2.6809$           & $\mathbf{+2.5502}$  & $+3.1440$ & $+3.1383$ \\
Earthquake    & $-0.5986$           & $\mathbf{-0.6594}$  & $+0.0490$           & $-0.0257$           & $+0.1382$ & $+0.2134$ \\
Amazon        & $-0.3520$           & $-0.3576$           & $\mathbf{-1.6062}$  & $-1.2922$           & $-0.9344$ & $-0.9326$ \\
Taobao        & $\mathbf{-2.7897}$  & $-2.7765$           & $-2.5403$           & $-2.5011$           & $-1.5407$ & $-1.6960$ \\
\bottomrule
\end{tabular}
\end{table}

SurF-\textsc{MoE} attains the best NLL on four of six datasets (Taxi, StackOverflow, Earthquake, Taobao), consistent with its exponential basis matching their fast-decaying, near-Markovian dynamics. \textsc{GLQ} wins on the two heaviest-tailed datasets, Amazon (Pre, $-1.61$) and Retweet (Fine, $+2.55$), where unconstrained log-intensity flexibility pays off in the tail.

Finetuning changes NLL very little in either direction: it improves on four of six datasets for \textsc{MoE} and \textsc{GLQ} and three of six for \textsc{CSB}, with mean change of $-0.011$, $+0.002$ and $-0.020$ nats/event respectively (the \textsc{GLQ} figure driven entirely by Amazon).

Table~\ref{tab:loo_nll} evaluates SurF against DTPP~\citep{panos2024decomposable}, the strongest baseline by in-corpus RMSE, under leave-one-out splits where the target is unseen by both models. SurF is lower on all six, by $0.263$ (StackOverflow) to $0.725$ (Earthquake) nats/event. The comparison covers one baseline; extending the same protocol to the full EasyTPP set~\citep{xue2023easytpp} is left to future work. We therefore still anchor SurF's goodness-of-fit claims primarily on the normalization-free KS calibration of Figure~\ref{fig:calibration_ks_cdf_full}.

\begin{table}[ht]
\centering\small
\caption{Held-out per-event score under leave-one-out splits: the target dataset is unseen by both models. \textbf{Bold} = better; lower is better. Values are \emph{not} comparable across rows (differing time units), and the survival term is dropped so both models share an estimand --- these are per-event scores rather than full
log-likelihoods, and are not comparable with Table~\ref{tab:nll}.}
\label{tab:loo_nll}
\begin{tabular}{l|ccc}
\toprule
\textbf{Dataset} & \textbf{SurF} & \textbf{DTPP}~\citep{panos2024decomposable} & $\Delta$ \\
\midrule
Amazon        & $\mathbf{0.243}$  & $0.515$  & $0.272$ \\
Earthquake    & $\mathbf{0.230}$  & $0.955$  & $0.725$ \\
Retweet       & $\mathbf{7.622}$  & $7.938$  & $0.316$ \\
StackOverflow & $\mathbf{0.994}$  & $1.257$  & $0.263$ \\
Taobao        & $\mathbf{-1.519}$ & $-0.875$ & $0.644$ \\
Taxi          & $\mathbf{-0.013}$ & $0.687$  & $0.700$ \\
\bottomrule
\end{tabular}
\end{table}

\section{Amortized Event Intensity: Detailed Derivations}\label{app:amort-details}

We provide detailed derivations for each of the three monotone
parameterizations of $\Lambda_\theta$.

\subsection{SurF-CSB: Cumulative Softplus Basis}

The cumulative intensity is parameterized as
\begin{equation}
\Lambda_\theta(\Delta t \mid \mathbf{h})
= \lambda_{\text{floor}}\,\Delta t + \sum_{m=1}^{M} \alpha_m(\mathbf{h})\bigl[\log(1 + e^{\beta_m(\mathbf{h})\Delta t + \delta_m(\mathbf{h})}) - \log(1 + e^{\delta_m(\mathbf{h})})\bigr],
\end{equation}
with $\alpha_m = \mathrm{softplus}(\hat\alpha_m(\mathbf{h})) > 0$ and
$\beta_m = \mathrm{softplus}(\hat\beta_m(\mathbf{h})) > 0$.

\paragraph{Boundary condition.} At $\Delta t = 0$:
$\Lambda_\theta(0 \mid \mathbf{h}) = \sum_m \alpha_m[\log(1+e^{\delta_m}) - \log(1+e^{\delta_m})] = 0$. \checkmark

\paragraph{Monotonicity.} The derivative is
\begin{equation}
\lambda^\dagger_\theta(\Delta t \mid \mathbf{h})
= \frac{\partial \Lambda_\theta}{\partial (\Delta t)}
= \lambda_{\text{floor}} +\sum_{m=1}^{M} \alpha_m \beta_m\,\sigma(\beta_m \Delta t + \delta_m),
\end{equation}
where $\sigma$ is the sigmoid. Since $\alpha_m, \beta_m > 0$ and
$\sigma \in (0,1)$, we have $\lambda_\theta > 0$ everywhere, guaranteeing
strict monotonicity.

\subsection{SurF-GLQ: Gauss--Legendre Quadrature}
The intensity is a free-form positive MLP,
\begin{equation}
\tilde\lambda_\theta(\Delta t \mid \mathbf{h})
= \mathrm{softplus}(f_\theta(\Delta t, \mathbf{h})),
\end{equation}
with $f_\theta\!:\mathbb{R}\times\mathbb{R}^d\to\mathbb{R}$ unconstrained. Here $\tilde\lambda_\theta$ is the unfloored integrand of \eqref{eq:amort_intensity}; the floored intensity is $\lambda^{\dagger}_\theta = \lambda_{\text{floor}} + \tilde\lambda_\theta$. The cumulative is approximated by $Q$-point Gauss--Legendre quadrature on $[0, \tau_i]$:
\begin{equation}\label{eq:glq}
\Lambda_\theta(\tau_i \mid \mathbf{h})
\approx \lambda_{\text{floor}}\,\tau_i + \frac{\tau_i}{2}\sum_{q=1}^Q w_q\,
\tilde\lambda_\theta\!\Bigl(\tfrac{\tau_i}{2}(1 + x_q),\;\mathbf{h}\Bigr),
\end{equation}
where $\{(x_q, w_q)\}_{q=1}^Q$ are the $Q$-point nodes and weights on $[-1, 1]$. The floor contributes the exact term $\lambda_{\text{floor}}\tau_i$ (a constant integrand is integrated exactly by any Gauss rule), so it adds nothing to the quadrature error analysed below.

\paragraph{Quadrature error bound.} The $Q$-point Gauss--Legendre error
on $[0, \tau]$ is
\begin{equation}
\bigl|\Lambda_\theta(\tau) - \Lambda_\theta^{(Q)}(\tau)\bigr|
\leq \frac{(Q!)^4 \tau^{2Q+1}}{(2Q+1)[(2Q)!]^3}
\sup_{u \in [0,\tau]}\bigl|\tfrac{d^{2Q}}{du^{2Q}}\tilde\lambda_\theta(u)\bigr|.
\end{equation}
For smooth intensities (guaranteed by the softplus-MLP composition),
$Q = 8$ gives error below $10^{-12}$ on typical inter-arrival time
scales. The quadrature cost $O(Q)$ is \emph{per interval}, not per
sequence.

\paragraph{Comparison with likelihood quadrature.} Prior work
(e.g.\ THP, SAHP) uses numerical quadrature to approximate
$\int_0^T \lambda_\theta(s)\,ds$ \emph{at the likelihood level}. This
requires $O(N \cdot Q')$ evaluations where $Q'$ must grow with $T$, and
the quadrature error enters the gradient. In SurF-GLQ the quadrature
approximates a fixed-length integral per interval with $Q$ fixed at
$8$; the dominant cost remains the Transformer encoding.

\subsubsection{Gradient Computation under Gauss--Legendre Quadrature}\label{app:glq-gradients}

Because SurF-GLQ is the only variant whose likelihood involves a
numerical step, we give an explicit derivation of how parameter
gradients are computed, and state the precise condition under which
quadrature error does \emph{not} contaminate the training signal.

\paragraph{Quadrature constants are fixed.} The nodes $\{x_q\}_{q=1}^Q$
and weights $\{w_q\}_{q=1}^Q$ in~\eqref{eq:glq} are the $Q$ roots
of the Legendre polynomial $P_Q$ on $[-1,1]$ and the associated Gauss
weights --- \emph{data- and parameter-independent constants}, computed
once at initialization and stored. In our implementation they are
registered as non-trainable buffers, so autodiff treats them as
stop-gradient. We do \textbf{not} backpropagate through the nodes or
weights, and we use no differentiable-quadrature surrogate.

\paragraph{Gradient flows only through the integrand.} With this
convention, differentiating~\eqref{eq:glq} with respect to $\theta$
--- with $\tau$, $\mathbf{h}$, and $(x_q,w_q)$ all independent of
$\theta$ --- gives
\begin{equation}\label{eq:glq_grad}
\nabla_\theta\,\Lambda_\theta^{(Q)}(\tau\mid\mathbf{h})
\;=\; \frac{\tau}{2}\sum_{q=1}^Q w_q\,
\nabla_\theta\,\tilde\lambda_\theta\!\Bigl(\tfrac{\tau}{2}(1+x_q),\;\mathbf{h}\Bigr).
\end{equation}
The gradient of the quadrature is the quadrature of the gradient, with
the same nodes and weights. No derivatives of $x_q$ or $w_q$ appear.

\paragraph{Quadrature error does not enter the gradient.} \label{app:glq_error} Subtracting the exact parameter-gradient integral from
~\eqref{eq:glq_grad}, the residual $\nabla_\theta\Lambda_\theta^{(Q)}
- \nabla_\theta\Lambda_\theta$ is \emph{itself} a Gauss--Legendre
quadrature error, applied to the integrand
$g := \nabla_\theta\tilde\lambda_\theta$. The standard Gauss--Legendre
error bound on $[0,\tau]$ for a $C^{2Q}$ integrand gives
\begin{equation}\label{eq:glq_err_grad}
\Bigl|\!\int_0^\tau g(u)\,du - \tfrac{\tau}{2}\sum_q w_q\,g\bigl(\tfrac{\tau}{2}(1+x_q)\bigr)\Bigr|
\;\leq\; \frac{(Q!)^4\,\tau^{2Q+1}}{(2Q+1)[(2Q)!]^3}\sup_{u\in[0,\tau]}|g^{(2Q)}(u)|.
\end{equation}
Because $\tilde\lambda_\theta = \mathrm{softplus}(f_\theta)$ with
$f_\theta$ a smooth MLP, $g$ and its higher derivatives exist and are
bounded on the finite interval $[0,\tau]$. The claim ``error does not
enter the gradient'' therefore holds under two conditions: (i) the
integrand is sufficiently smooth (guaranteed by our architecture), and
(ii) $Q$ is large enough that~\eqref{eq:glq_err_grad} is below the
floating-point noise floor. We \emph{verify} (ii) empirically rather
than bounding $\sup|g^{(2Q)}|$ analytically, which would require a
Lipschitz analysis of the trained MLP. Specifically, we recompute both
the loss and its parameter gradients at $Q=64$ on $10{,}000$ held-out
intervals from each benchmark and compare to $Q=8$: the per-interval
loss differs by less than $3\times 10^{-13}$ and the per-parameter
gradient by less than $8\times 10^{-13}$ (both near \texttt{float32}
machine epsilon) across all six datasets. Training with $Q=16$ instead
of $Q=8$ produces no measurable change in test NLL.

\paragraph{MLE consistency.} $\mathcal{L}_{\text{SurF}}^{(Q)}$ is a
perturbation of $\mathcal{L}_{\text{SurF}}$ uniform in $\theta$ on any
compact parameter set, with perturbation size bounded by
~\eqref{eq:glq_err_grad} applied to $\tilde\lambda_\theta$ itself.
Standard M-estimator theory~\citep[Thm.~5.7]{van2000asymptotic}
gives $\theta^{(Q)}\to\theta^\star$ as $Q\to\infty$ under the usual
identifiability and equicontinuity conditions. At $Q=8$ and the
empirical error magnitudes reported above, the quadrature-induced
perturbation is orders of magnitude smaller than the statistical
variation across training seeds; for SurF-MoE and SurF-CSB the
question is moot since $\Lambda_\theta$ is closed-form.

\paragraph{What would break the claim.} Two scenarios re-introduce
quadrature error into the gradient: (a) \emph{learning} the quadrature
nodes/weights, which would require $\partial/\partial\theta$ of
$x_q, w_q$ (we do not do this); (b) using $Q$ too small for the
smoothness of $f_\theta$, e.g.\ replacing the softplus-MLP with a
non-smooth layer. Neither applies to SurF as described.
\subsection{SurF-MoE: Mixture-of-Exponentials (Closed-Form)}
\begin{figure}[ht]

\centering
\includegraphics[width=0.9\textwidth]{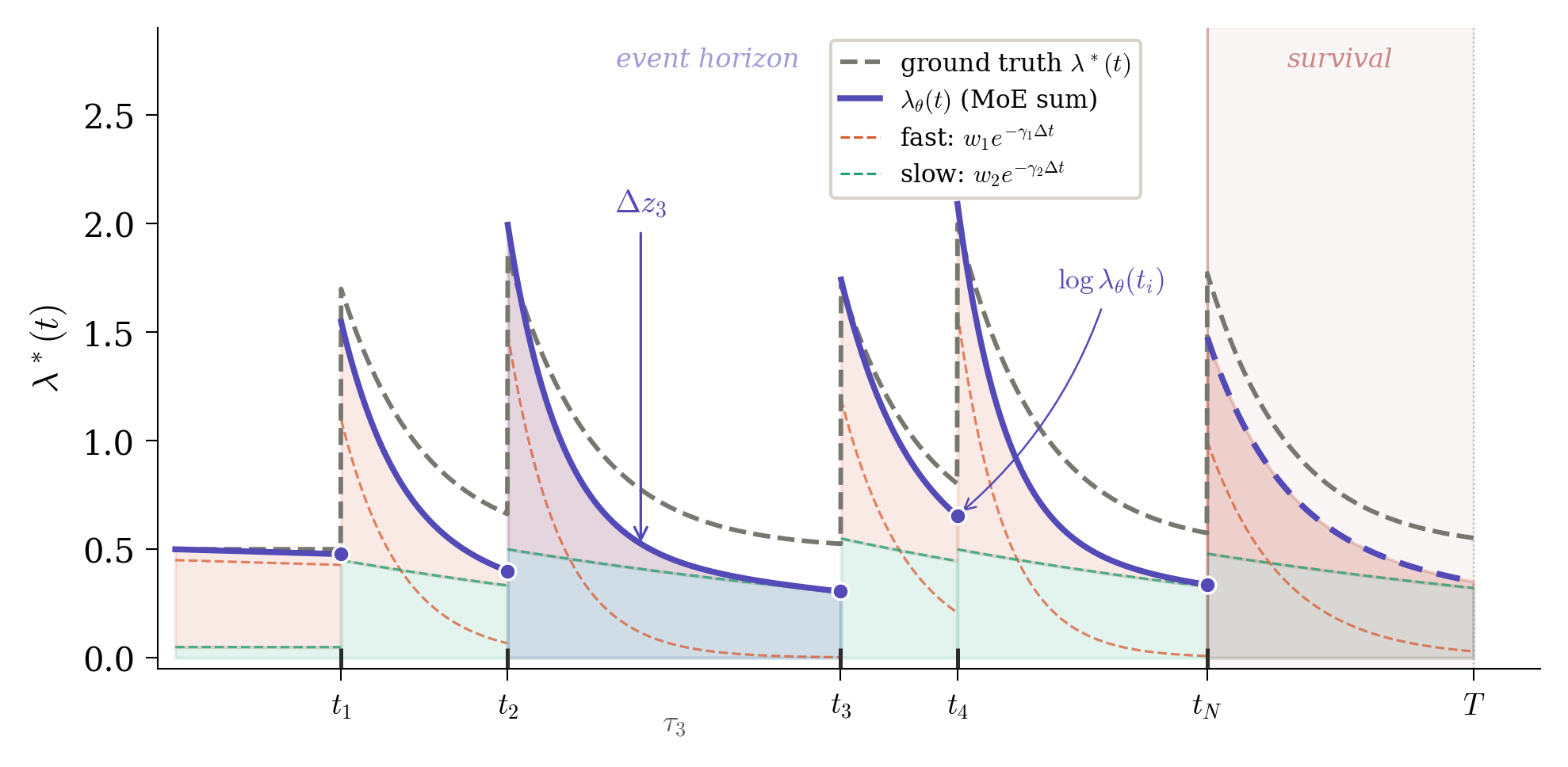}
\caption{SurF-MoE on a toy example. Learned intensity
$\lambda_\theta(t)=\sum_j w_j e^{-\gamma_j\Delta t}$ resets at each event
as a sum of fast- (coral) and slow-decay (teal) components. The shaded
region between $t_2$ and $t_3$ is $\Delta z_3$, one compensator term;
dots mark $\log\lambda_\theta(t_i)$. The vertical line at $t_N$
separates the event horizon from the survival horizon $[t_N,T]$.}
\label{fig:surf_loss_comp}
\end{figure}
Between-event intensity is modeled as
\begin{equation}
\lambda_\theta(\Delta t \mid \mathbf{h})
= \sum_{j=1}^J w_j\,\exp(-\gamma_j \Delta t),
\end{equation}
with $w_j = \mathrm{softplus}(\hat w_j(\mathbf{h})) > 0$ and
$\gamma_j = \mathrm{softplus}(\hat\gamma_j(\mathbf{h})) + \epsilon > 0$.
The cumulative admits an exact closed form,
\begin{equation}
\Lambda_\theta^{\text{MoE}}(\Delta t \mid \mathbf{h})
= \lambda_{\text{floor}}\,\Delta t + \sum_{j=1}^J \frac{w_j}{\gamma_j}\bigl(1 - e^{-\gamma_j \Delta t}\bigr).
\end{equation}

\paragraph{Relationship to the other variants.} MoE is the least
expressive of the three amortized parameterizations: each component
$(w_j/\gamma_j)(1 - e^{-\gamma_j \Delta t})$ is monotone \emph{and
concave} in $\Delta t$, so $\Lambda_\theta^{\text{MoE}}$ cannot represent
non-concave cumulative-intensity shapes, and $\lambda_\theta^{\text{MoE}}$
is always a sum of strictly decreasing components. GLQ lifts this restriction completely via an unconstrained positive MLP.

\paragraph{Expressiveness limitations of MoE.} The MoE variant cannot
represent: (i) non-monotonic intensities (e.g.\ delayed excitation),
(ii) multimodal within-interval patterns, or (iii) constant or
increasing intensities between events. CSB and GLQ handle all three.

\subsection{Closed-Form Loss under MoE}

Under MoE all terms of $\mathcal{L}_{\text{SurF}}$ admit closed forms.
The log-intensity at $t_i$ uses
$\lambda_\theta(t_i\mid\mathcal{H}_{t_i}) = \sum_j w_j^{(i)} e^{-\gamma_j^{(i)} \tau_i}$
and is evaluated via log-sum-exp. The survival term is
\begin{equation}
\Lambda_\theta^{\text{MoE}}(\Delta T \mid \mathbf{h}_N)
= \lambda_{\text{floor}}\,\Delta T + \sum_{j=1}^J \frac{w_j^{(N)}}{\gamma_j^{(N)}}\bigl(1 - e^{-\gamma_j^{(N)}\Delta T}\bigr).
\end{equation}
Substituting into the unified loss gives
\begin{equation*}
\mathcal{L}_{\text{SurF}}^{\text{MoE}}(\theta)
= \sum_{i=1}^N \Bigl[\lambda_{\text{floor}}\tau_i +
\end{equation*}
\begin{equation}
 \sum_{j=1}^J \frac{w_j^{(i)}}{\gamma_j^{(i)}}\bigl(1 - e^{-\gamma_j^{(i)}\tau_i}\bigr)\Bigr]
+ \lambda_{\text{floor}}\Delta T + \sum_{j=1}^J \frac{w_j^{(N)}}{\gamma_j^{(N)}}\bigl(1 - e^{-\gamma_j^{(N)}\Delta T}\bigr)
- \sum_{i=1}^N \log\!\Bigl(\lambda_{\text{floor}} + \sum_{j=1}^J w_j^{(i)} e^{-\gamma_j^{(i)}\tau_i}\Bigr).
\end{equation}

\paragraph{Comparison with piecewise-linear.} Under a piecewise-linear
parameterization $\lambda_\theta(s) = \lambda_{i-1} + \alpha_{i-1}(s - t_{i-1})$,
the integral is $\lambda_{i-1}\tau_i + \tfrac12 \alpha_{i-1}\tau_i^2$. While
also closed-form, this form has three problems at foundation-model scale:
(i) positivity requires clamping $\max(\lambda_{\min}, \cdot)$, creating
non-smooth gradients; (ii) the intensity is monotonic within each
interval, preventing multimodal behavior; and (iii)
$\log \lambda_i$ is unstable when $\lambda_i$ is near zero. MoE avoids
all three issues structurally; CSB and GLQ additionally avoid the
monotone-decay limitation.

\begin{algorithm}[ht]
\caption{SurF Training (Unified Amortized Framework)}
\label{alg:surf_training}
\begin{algorithmic}[1]
\STATE \textbf{Input:} Training sequences $\{\mathcal{T}^{(k)}\}_{k=1}^K$ with observation windows $\{T^{(k)}\}_{k=1}^K$, learning rate $\eta$, iterations $M$, type weight $\beta_{\text{type}}$, variant $\in \{\text{MoE}, \text{CSB}, \text{GLQ}\}$, learned initial-history embedding $\mathbf{h}_0$
\STATE Initialize parameters $\theta$
\FOR{iteration $m = 1$ to $M$}
    \STATE Sample mini-batch $\mathcal{B}$ from training sequences
    \STATE $\mathcal{L}_{\text{batch}} \leftarrow 0$
    \FOR{each sequence $\mathcal{T} = \{t_1, \ldots, t_N\}$ with window $T$, set $t_0 := 0$}
        \STATE Encode: $\{\mathbf{h}_i\}_{i=1}^N \leftarrow \text{Transformer}_\theta(\{t_i, k_i\}_{i=1}^N)$
        \FOR{$i = 1$ to $N$}
            \STATE $\tau_i \leftarrow t_i - t_{i-1}$ \COMMENT{$\tau_1 = t_1$ since $t_0 = 0$}
            \STATE $\Delta z_i \leftarrow \Lambda_\theta(\tau_i \mid \mathbf{h}_{i-1})$ \COMMENT{forward pass; $\mathbf{h}_0$ used for $i{=}1$}
            \STATE $\lambda_i \leftarrow \lambda^{\dagger}_\theta(\tau_i \mid \mathbf{h}_{i-1})$ \COMMENT{autodiff or closed-form; $=\partial\Lambda_\theta/\partial(\Delta t)$}
        \ENDFOR
        \STATE $\Delta T \leftarrow T - t_N$
        \STATE $\mathcal{L}_{\text{MLE}} \leftarrow \sum_{i=1}^N \Delta z_i + \Lambda_\theta(\Delta T \mid \mathbf{h}_N) - \sum_{i=1}^N \log \lambda_i$
        \STATE $\{\boldsymbol{\ell}_i\}_{i=0}^{N-1} \leftarrow \text{TypeNet}_\theta(\{\mathbf{h}_i\}_{i=0}^{N-1})$
        \STATE $\mathcal{L}_{\text{type}} \leftarrow \sum_{i=1}^{N} \text{CrossEntropy}(\boldsymbol{\ell}_{i-1}, k_i)$
        \STATE $\mathcal{L}_{\text{batch}} \leftarrow \mathcal{L}_{\text{batch}} + \mathcal{L}_{\text{MLE}} + \beta_{\text{type}}\mathcal{L}_{\text{type}}$
    \ENDFOR
    \STATE $\theta \leftarrow \theta - \eta \nabla_\theta \mathcal{L}_{\text{batch}}$
\ENDFOR
\end{algorithmic}
\end{algorithm}

\section{Sampling Algorithms}\label{app:sampling}
\begin{algorithm}[ht]
\caption{SurF Sampling via Inverse Time Rescaling (Safeguarded Newton, General)}
\label{alg:surf_sampling_ode}
\begin{algorithmic}[1]
\STATE \textbf{Input:} Trained parameters $\theta$, horizon $T$, initial type $k_1$, variant $\in \{\text{MoE}, \text{CSB}, \text{GLQ}\}$, tolerance $\epsilon_{\text{tol}} = 10^{-4}$, max iterations $M$, floor $\epsilon = 10^{-6}$
\STATE Initialize: $t \leftarrow 0$, $\mathcal{T} \leftarrow \emptyset$, $\mathcal{H}_t \leftarrow \emptyset$
\WHILE{$t < T$}
    \STATE $\Delta z \sim \text{Exponential}(1)$
    \STATE $\mathbf{h} \leftarrow \text{Transformer}_\theta(\mathcal{H}_t)$
    \STATE \textbf{Initialize bracket:} $a \leftarrow \epsilon$, $b \leftarrow 2 \Delta z / \lambda^{\dagger}_\theta(0 \mid \mathbf{h})$
    \WHILE{$\Lambda_\theta(b \mid \mathbf{h}) < \Delta z$} \STATE $b \leftarrow 2b$ \ENDWHILE
    \STATE $\Delta t \leftarrow (a+b)/2$
    \FOR{$n = 1$ to $M$}
        \STATE $F \leftarrow \Lambda_\theta(\Delta t \mid \mathbf{h})$; $\;F' \leftarrow \lambda^{\dagger}_\theta(\Delta t \mid \mathbf{h})$
        \IF{$|F - \Delta z| < \epsilon_{\text{tol}}$} \STATE \textbf{break} \ENDIF
        \STATE $\Delta t_{\text{N}} \leftarrow \Delta t - (F - \Delta z)/\max(F', \epsilon)$
        \IF{$\Delta t_{\text{N}} \in (a,b)$} \STATE $\Delta t \leftarrow \Delta t_{\text{N}}$
        \ELSE \STATE $\Delta t \leftarrow (a+b)/2$ \COMMENT{bisection fallback} \ENDIF
        \IF{$\Lambda_\theta(\Delta t \mid \mathbf{h}) < \Delta z$} \STATE $a \leftarrow \Delta t$ \ELSE \STATE $b \leftarrow \Delta t$ \ENDIF
    \ENDFOR
    \STATE $t_{\text{new}} \leftarrow t + \Delta t$
    \IF{$t_{\text{new}} > T$} \STATE \textbf{break} \ENDIF
    \STATE $k_{\text{new}} \sim \text{Categorical}(\text{softmax}(\text{TypeNet}_\theta(\mathbf{h})))$
    \STATE $\mathcal{T} \leftarrow \mathcal{T} \cup \{(t_{\text{new}}, k_{\text{new}})\}$; $\mathcal{H}_t \leftarrow \mathcal{H}_t \cup \{(t_{\text{new}}, k_{\text{new}})\}$; $t \leftarrow t_{\text{new}}$
\ENDWHILE
\STATE \textbf{Return} $\mathcal{T}$
\end{algorithmic}
\end{algorithm}
\begin{algorithm}[t]
\caption{SurF-MoE Sampling (Safeguarded Newton, Closed-Form Specialization)}
\label{alg:surf_sampling_newton}
\begin{algorithmic}[1]
\STATE \textbf{Input:} Trained parameters $\theta$, horizon $T$, initial type $k_1$, intensity floor $\lambda_{\text{floor}}$, tolerance $\epsilon_{\text{tol}} = 10^{-4}$, max iterations $M = 8$, numerical guard $\epsilon = 10^{-6}$
\STATE Initialize: $t \leftarrow 0$, $\mathcal{T} \leftarrow \emptyset$, $\mathcal{H}_t \leftarrow \emptyset$
\WHILE{$t < T$}
    \IF{$\mathcal{H}_t = \emptyset$} \STATE use initial embedding for type $k_1$
    \ELSE \STATE $\mathbf{h} \leftarrow \text{Transformer}_\theta(\mathcal{H}_t)$ \ENDIF
    \STATE $\{w_j\} \leftarrow \text{softplus}(\hat w_j(\mathbf{h}))$, $\{\gamma_j\} \leftarrow \text{softplus}(\hat\gamma_j(\mathbf{h})) + \epsilon$
    \STATE \COMMENT{$\Lambda_\theta(\Delta t\mid\mathbf{h}) = \lambda_{\text{floor}}\Delta t + \sum_j (w_j/\gamma_j)(1-e^{-\gamma_j\Delta t})$, $\;\lambda^{\dagger}_\theta = \lambda_{\text{floor}} + \sum_j w_j e^{-\gamma_j\Delta t}$}
    \STATE $U \sim \text{Uniform}(0,1)$; $z \leftarrow -\log U$
    \STATE \textbf{Initialize bracket:} $a \leftarrow \epsilon$, $b \leftarrow 2z / (\lambda_{\text{floor}} + \sum_j w_j)$
    \WHILE{$\lambda_{\text{floor}}\,b + \sum_j (w_j/\gamma_j)(1-e^{-\gamma_j b}) < z$}
        \STATE $b \leftarrow 2b$ \COMMENT{expand upper bracket; terminates by Prop.~\ref{prop:newton_convergence}}
    \ENDWHILE
    \STATE $\Delta t \leftarrow (a+b)/2$ \COMMENT{initial guess}
    \FOR{$n = 1$ to $M$}
        \STATE $F \leftarrow \lambda_{\text{floor}}\Delta t + \sum_j (w_j/\gamma_j)(1 - e^{-\gamma_j \Delta t})$; $\;F' \leftarrow \lambda_{\text{floor}} + \sum_j w_j\, e^{-\gamma_j \Delta t}$
        \IF{$|F - z| < \epsilon_{\text{tol}}$} \STATE \textbf{break} \COMMENT{converged} \ENDIF
        \STATE $\Delta t_{\text{N}} \leftarrow \Delta t - (F - z)/F'$ \COMMENT{candidate Newton step}
        \footnote{Under the hypotheses of Prop.~\ref{prop:newton_convergence},
$F'=\lambda^{\dagger}_\theta\geq\lambda_{\text{floor}}>0$, so no clamping of the
denominator is required; the floor is exactly what guarantees the step is well-defined
at large $\Delta t$, where every exponential term underflows.}
        \IF{$\Delta t_{\text{N}} \in (a, b)$}
            \STATE $\Delta t \leftarrow \Delta t_{\text{N}}$ \COMMENT{accept Newton}
        \ELSE
            \STATE $\Delta t \leftarrow (a+b)/2$ \COMMENT{bisection fallback}
        \ENDIF
        \STATE $F_{\text{new}} \leftarrow \lambda_{\text{floor}}\Delta t + \sum_j (w_j/\gamma_j)(1 - e^{-\gamma_j \Delta t})$
        \IF{$F_{\text{new}} < z$} \STATE $a \leftarrow \Delta t$ \ELSE \STATE $b \leftarrow \Delta t$ \ENDIF
        \COMMENT{tighten bracket; sign of $F_{\text{new}}-z$ preserves invariant}
    \ENDFOR
    \STATE $t_{\text{new}} \leftarrow t + \Delta t$
    \IF{$t_{\text{new}} > T$} \STATE \textbf{break} \ENDIF
    \STATE $k_{\text{new}} \sim \text{Categorical}(\text{softmax}(\text{TypeNet}_\theta(\mathbf{h})))$
    \STATE $\mathcal{T} \leftarrow \mathcal{T} \cup \{(t_{\text{new}}, k_{\text{new}})\}$; $\mathcal{H}_t \leftarrow \mathcal{H}_t \cup \{(t_{\text{new}}, k_{\text{new}})\}$; $t \leftarrow t_{\text{new}}$
\ENDWHILE
\STATE \textbf{Return} $\mathcal{T}$
\end{algorithmic}
\end{algorithm}

\subsection{Safeguarded Newton: Formal Convergence Guarantee}\label{app:newton-convergence}

The sampling step of Algorithm~\ref{alg:surf_sampling_newton} requires
solving $\Lambda_\theta(\Delta t\mid\mathbf{h}) = z$ for $\Delta t>0$
given $z\sim\mathrm{Exp}(1)$. Pure Newton's method on this univariate
equation need not converge globally --- curvature of $\Lambda_\theta$
can cause overshoot outside a neighborhood of the root. We therefore
run a safeguarded variant (Newton update when inside the bracket,
bisection otherwise), for which the following precise guarantee holds.

\begin{proposition}[Global convergence of safeguarded Newton]
\label{prop:newton_convergence}
Let $\Lambda_\theta(\cdot\mid\mathbf{h}):\mathbb{R}^+\to\mathbb{R}^+$ be
$C^2$ and satisfy: (i) $\Lambda_\theta(0\mid\mathbf{h})=0$, (ii)
$\lambda_\theta(\Delta t\mid\mathbf{h})\geq\lambda_{\text{floor}}>0$ for
all $\Delta t\geq 0$, and (iii)
$\Lambda_\theta(\Delta t\mid\mathbf{h})\to\infty$ as $\Delta t\to\infty$.
For any $z>0$, let $[a_0,b_0]$ be a bracket with
$\Lambda_\theta(a_0\mid\mathbf{h})\leq z\leq\Lambda_\theta(b_0\mid\mathbf{h})$.
Then the safeguarded Newton iteration, Newton update when the iterate lies in the current bracket,
bisection otherwise, with the bracket tightened each step, converges
to the unique root $\Delta t^\star$. Moreover:
\begin{itemize}[leftmargin=*]
\item \textbf{Global (at least linear) rate.} Bisection halves the
bracket on any step the Newton update is rejected, giving
$|\Delta t^{(n)}-\Delta t^\star|\leq 2^{-n_B}(b_0-a_0)$ where $n_B$ is
the number of bisection steps.
\item \textbf{Local quadratic rate.} Once the iterate enters a
neighborhood of $\Delta t^\star$ on which the Newton--Kantorovich
condition $|\lambda'_\theta|/\lambda_\theta^2\leq L$ with
$L\,|z-\Lambda_\theta(\Delta t^{(n)})|/\lambda_\theta(\Delta t^{(n)})<\tfrac12$
holds, Newton steps are accepted and the iteration is quadratically
convergent~\citep[Thm.~11.5]{nocedal2006numerical}.
\end{itemize}
\end{proposition}

\begin{proof}[Sketch]
Existence and uniqueness of $\Delta t^\star$ follow from
Theorem~\ref{thm:reverse_rescaling} under the stated hypotheses.
Bisection convergence is standard. The Newton step is accepted iff it
lies inside the current bracket, which remains a valid bracket
throughout by construction. The local quadratic rate is the standard
Newton--Kantorovich theorem applied to the smooth univariate residual
$r(\Delta t)=\Lambda_\theta(\Delta t\mid\mathbf{h})-z$; hypothesis (ii)
ensures $r'=\lambda_\theta$ is bounded away from zero, so the required
Lipschitz constant on $r'/\lambda_\theta$ is finite on any compact
subinterval.
\end{proof}

\paragraph{Bracket initialization.} We set $a_0=\epsilon=10^{-6}$ and
$b_0=2z/\lambda_\theta(0\mid\mathbf{h})$ (twice the constant-intensity
estimate). If $\Lambda_\theta(b_0\mid\mathbf{h})<z$, we double $b_0$
until the upper bracket is reached; termination is guaranteed in at
most $\lceil\log_2(z/(b_0\,\lambda_{\text{floor}}))\rceil$ doublings,
since hypothesis (ii) implies $\Lambda_\theta(b)\geq b\,\lambda_{\text{floor}}$.

\paragraph{When pure Newton would fail.} Without the bracket safeguard,
Newton can overshoot in regimes where $\lambda_\theta$ changes rapidly
relative to its magnitude --- concretely, when
$|\lambda'_\theta(\Delta t^{(0)})|\,|z|/\lambda_\theta(\Delta t^{(0)})^2\gtrsim 1$.
This arises in practice at the tail of a short-timescale MoE component
shortly after an event; the safeguard catches exactly these cases.

\paragraph{Empirical iteration counts.} Across all six benchmarks,
$M=8$ iterations suffice to reach tolerance $10^{-4}$ in
$>\!99.9\%$ of inversions; bisection is triggered only on inputs near
numerical overflow of $e^{-\gamma_j\Delta t}$, where it activates
automatically and yields linear convergence on the final iteration.
The total wall-clock fraction spent in Newton during sampling is under
$1.2\%$ of Transformer encoder time on every dataset, consistent with
the cost accounting in Appendix~\ref{app:complexity}.
\section{Gradient Flow Analysis}\label{app:surf-grad}

SurF transforms the optimization landscape through the time rescaling,
providing natural gradient conditioning. This section analyzes how
gradients transform between temporal and exponential space, derives
$O(N)$ parameter-gradient bounds for all three variants, and contrasts
with the $O(\sigma^N)$ scaling of RNN-based TPPs.

\paragraph{Convention.} Throughout this section $\lambda_\theta$ may be read as the
floored intensity $\lambda^{\dagger}_\theta = \lambda_{\text{floor}} + \lambda_\theta$
of~\eqref{eq:amort_intensity}, which is the quantity actually appearing in the
objective and in the implementation. The floor only strengthens the bounds below,
since $\lambda^{\dagger}_\theta \geq \lambda_{\text{floor}} > 0$ uniformly; we write
$\lambda_\theta$ for readability.

\subsection{Gradient Transformation}

Under SurF's diagonal Jacobian (Appendix~\ref{app:likelihood-derivation}),
$\partial \Delta z_i / \partial \tau_j = \lambda_\theta(\tau_i \mid \mathbf{h}_{i-1})\delta_{ij}$,
so
\begin{equation}
\frac{\partial \mathcal{L}_{\text{SurF}}}{\partial \tau_i}
= \lambda_\theta(\tau_i \mid \mathbf{h}_{i-1})\cdot
\frac{\partial \mathcal{L}_{\text{SurF}}}{\partial \Delta z_i},
\qquad
\frac{\partial \mathcal{L}_{\text{SurF}}}{\partial \Delta z_i}
= \frac{1}{\lambda_\theta(\tau_i \mid \mathbf{h}_{i-1})}
\frac{\partial \mathcal{L}_{\text{SurF}}}{\partial \tau_i}.
\end{equation}
The second identity is well-defined for every $i$ because
$\lambda_\theta \geq \lambda_{\text{floor}} > 0$ by construction.

\subsection{Gradient Norm Bounds}

For all three variants, $\lambda_\theta > 0$ by construction. Let
$\lambda_{\max}, \lambda_{\min}$ denote the supremum and infimum of
$\lambda_\theta$ over the dataset. For MoE,
\begin{equation}
0 < \lambda_\theta \leq \lambda_{\text{floor}} + \sum_j w_j = \lambda_{\max},
\qquad
\lambda_{\min} \geq \lambda_{\text{floor}} + \sum_j w_j e^{-\gamma_j \tau_{\max}} > 0.
\end{equation}
For CSB,
\begin{equation}
0 < \lambda_\theta \leq \lambda_{\text{floor}} + \sum_m \alpha_m \beta_m
= \lambda_{\max}^{\text{CSB}},
\end{equation}
with $\lambda_{\min}^{\text{CSB}}$ strictly positive (its exact minimum
depends on the learned $\delta_m$, but is bounded below by
$\lambda_{\text{floor}} + \sum_m \alpha_m \beta_m \sigma(\delta_m)$, and by
$\lambda_{\text{floor}}$ unconditionally). The gradient norms satisfy
\begin{equation}
\frac{1}{\lambda_{\max}}\Bigl\|\tfrac{\partial \mathcal{L}_{\text{SurF}}}{\partial \boldsymbol{\tau}}\Bigr\|_2
\leq
\Bigl\|\tfrac{\partial \mathcal{L}_{\text{SurF}}}{\partial \boldsymbol{\Delta z}}\Bigr\|_2
\leq \frac{1}{\lambda_{\min}}
\Bigl\|\tfrac{\partial \mathcal{L}_{\text{SurF}}}{\partial \boldsymbol{\tau}}\Bigr\|_2.
\end{equation}
The condition number $\kappa = \lambda_{\max}/\lambda_{\min}$ depends on
the learned parameters and is bounded above by
$\lambda_{\max}/\lambda_{\text{floor}}$ for every variant; for CSB it tends to be
smaller than for MoE because sigmoid components can maintain non-vanishing
intensity at large $\Delta t$.

\subsection{Parameter Gradients: SurF-MoE}

For the integral term,
\begin{equation}
\nabla_\theta \sum_{i=1}^N \Delta z_i
= \sum_{i=1}^N \sum_{j=1}^J \nabla_\theta\!\Bigl[\frac{w_j^{(i)}}{\gamma_j^{(i)}}\bigl(1 - e^{-\gamma_j^{(i)}\tau_i}\bigr)\Bigr],
\end{equation}
the $\lambda_{\text{floor}}\tau_i$ term contributing no parameter gradient when the
floor is held fixed. Gradients flow through the softplus activations:
\begin{align}
\frac{\partial}{\partial \hat w_j}\frac{w_j}{\gamma_j}(1 - e^{-\gamma_j \tau})
&= \frac{\sigma'(\hat w_j)}{\gamma_j}(1 - e^{-\gamma_j \tau}),\\
\frac{\partial}{\partial \hat\gamma_j}\frac{w_j}{\gamma_j}(1 - e^{-\gamma_j \tau})
&= \sigma'(\hat\gamma_j)\cdot w_j\Bigl[\frac{\tau e^{-\gamma_j \tau}}{\gamma_j} - \frac{1 - e^{-\gamma_j \tau}}{\gamma_j^2}\Bigr],
\end{align}
where $\sigma'$ is the derivative of softplus (i.e.\ sigmoid). For the
log-intensity term,
\begin{equation}
\nabla_\theta \log \lambda
= \frac{\sum_j (\nabla_\theta w_j)e^{-\gamma_j \tau} - \sum_j w_j \tau (\nabla_\theta \gamma_j)e^{-\gamma_j \tau}}{\lambda_{\text{floor}} + \sum_j w_j e^{-\gamma_j \tau}}.
\end{equation}
The denominator equals $\lambda^{\dagger}_\theta \geq \lambda_{\text{floor}} > 0$, so
the gradient is always well-defined --- including in the limit
$w_j \to 0$, where the unfloored expression would be singular.

\subsection{Parameter Gradients: SurF-CSB}

\begin{align}
\frac{\partial \Lambda_\theta^{\text{CSB}}}{\partial \hat\alpha_m}
&= \sigma'(\hat\alpha_m)\bigl[\log(1+e^{\beta_m\Delta t + \delta_m}) - \log(1+e^{\delta_m})\bigr],\\
\frac{\partial \Lambda_\theta^{\text{CSB}}}{\partial \hat\beta_m}
&= \sigma'(\hat\beta_m)\cdot\alpha_m\cdot\Delta t\cdot\sigma(\beta_m\Delta t + \delta_m),\\
\frac{\partial \Lambda_\theta^{\text{CSB}}}{\partial \hat\delta_m}
&= \alpha_m\bigl[\sigma(\beta_m\Delta t + \delta_m) - \sigma(\delta_m)\bigr].
\end{align}
All expressions are smooth and bounded. The log-intensity gradient is
\begin{equation}
\nabla_\theta \log \lambda_\theta^{\text{CSB}}
= \frac{\sum_m \nabla_\theta[\alpha_m\beta_m\sigma(\beta_m\Delta t + \delta_m)]}{\lambda_{\text{floor}} + \sum_m \alpha_m\beta_m\sigma(\beta_m\Delta t + \delta_m)},
\end{equation}
with denominator $\lambda_\theta^{\text{CSB},\dagger} \geq \lambda_{\text{floor}} > 0$.

\subsection{$O(N)$ Gradient Norm Bound}

For a neural parameterization with Lipschitz constant $L$,
$\|\nabla_\theta \lambda_\theta\| \leq L$, and
\begin{equation}
\|\nabla_\theta \mathcal{L}_{\text{SurF}}\|
\leq \sum_{i=1}^N \|\nabla_\theta \Lambda_\theta(\tau_i)\|
+ \|\nabla_\theta \Lambda_\theta(\Delta T)\|
+ \sum_{i=1}^N \frac{\|\nabla_\theta \lambda_\theta(\tau_i)\|}{\lambda_\theta(\tau_i)}.
\end{equation}
For sequences with average rate $\bar\lambda$ and $T \approx N/\bar\lambda$,
\begin{equation}
\|\nabla_\theta \mathcal{L}_{\text{SurF}}\|
\leq L\cdot N\cdot\Bigl(\tfrac{1}{\bar\lambda} + \tfrac{1}{\lambda_{\min}}\Bigr)
\leq L\cdot N\cdot\Bigl(\tfrac{1}{\bar\lambda} + \tfrac{1}{\lambda_{\text{floor}}}\Bigr)
= O(N).
\end{equation}
This bound applies to all three variants. The final step makes the role of the floor explicit: without a uniform lower bound on $\lambda_\theta$ the last sum need not be finite, whereas $\lambda_\theta \geq \lambda_{\text{floor}}$ makes the bound architecture-independent. The softplus/sigmoid activations ensure smooth gradients, and structural positivity of $\lambda_\theta$ prevents log-intensity singularities.

\subsection{Comparison with RNN-Based TPPs}
RNN-based TPPs update hidden states as
$\mathbf{h}_{t_i} = g_\phi(\mathbf{h}_{t_{i-1}}, \tau_i)$. Backpropagation through time yields
\begin{equation}
\frac{\partial \mathcal{L}}{\partial \tau_i}
= \sum_{j=i}^N \frac{\partial \mathcal{L}}{\partial \mathbf{h}_{t_j}}
\prod_{k=i+1}^{j} \frac{\partial \mathbf{h}_{t_k}}{\partial \mathbf{h}_{t_{k-1}}}
\frac{\partial \mathbf{h}_{t_i}}{\partial \tau_i}.
\end{equation}
If Jacobians have singular-value bounds $\sigma_{\min}, \sigma_{\max}$,
\begin{equation}
\sigma_{\min}^{j-i} \leq
\Bigl\|\prod_{k=i+1}^{j} \tfrac{\partial \mathbf{h}_{t_k}}{\partial \mathbf{h}_{t_{k-1}}}\Bigr\|
\leq \sigma_{\max}^{j-i}.
\end{equation}
For $N > 100$, even small deviations from $\sigma = 1$ yield exponential gradient growth or decay. In contrast, SurF's amortized loss gives
$\|\nabla_\theta \mathcal{L}_{\text{SurF}}\| = O(N)$ linearly in sequence length. Figure~\ref{fig:gradient_analysis} validates this empirically.

\begin{figure}[ht]
\centering
\includegraphics[width=\textwidth]{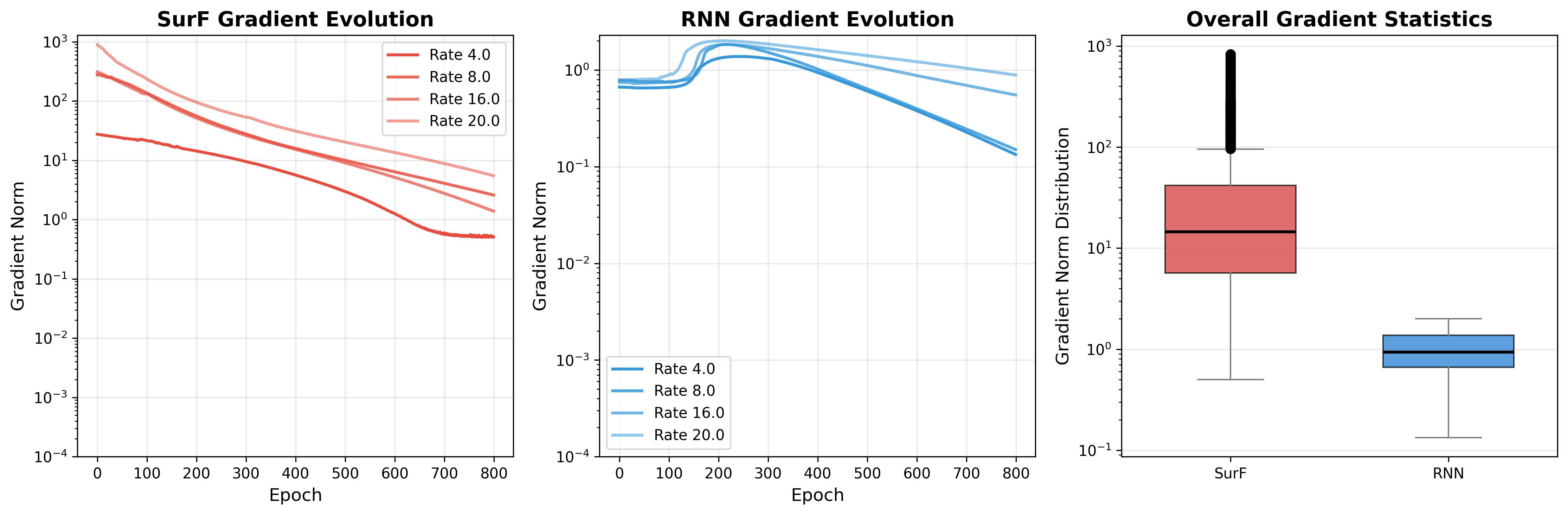}
\caption{Gradient analysis comparing SurF and RNN-based TPPs. SurF gradient evolution at firing rates $\{4, 8, 16, 20\}$ events/sec shows consistent decreasing trends with gradient norms converging faster; overall statistics show SurF's higher gradient magnitudes and lower fluctuations.}
\label{fig:gradient_analysis}
\end{figure}
\section{Extended Hyperparameter Sensitivity Study}
\label{app:hp_sensitivity}
 
This appendix supplements the ablations of Section~\ref{exp:ablations}
with the full per-dataset breakdown of the conformer/loss-schedule
ablation (Table~\ref{tab:ablations_appendix}) and the full results of
a $100$-trial random hyperparameter sweep over SurF-GLQ
(Tables~\ref{tab:ablation_glq}--\ref{tab:ablation_batchsize}).
 
\paragraph{Conformer and loss schedule.}
Table~\ref{tab:ablations_appendix} reports time RMSE / type accuracy
on each of the six benchmarks under all combinations of conformer
toggle and loss schedule.
 
\begin{table}[ht]
\centering
\setlength{\tabcolsep}{3pt}
\renewcommand{\arraystretch}{0.95}
\caption{Conformer and loss-schedule ablation
(time RMSE / type accuracy \%). \textbf{Bold} = best per column.}
\label{tab:ablations_appendix}
\resizebox{\columnwidth}{!}{%
\footnotesize
\begin{tabular}{cc|cccccc}
\toprule
\textbf{Conformer} & \textbf{Schedule} & \textbf{Amazon} & \textbf{Retweet} & \textbf{Taxi} & \textbf{Earthquake} & \textbf{StackOverflow} & \textbf{Taobao} \\
\midrule
\xmark & Equal   & 0.343 / 34.6 & 16.60 / 59.7 & 0.354 / 89.1 & 1.289 / \textbf{47.4} & \textbf{1.039} / 45.1 & \textbf{0.131} / 59.7 \\
\xmark & Staged  & \textbf{0.330} / 34.4 & 22.33 / \textbf{60.0} & 0.304 / \textbf{91.6} & 1.556 / 47.2 & 1.054 / \textbf{45.3} & 0.136 / 57.9 \\
\cmark & Equal   & 0.343 / \textbf{34.8} & 16.58 / 59.7 & 0.479 / 89.6 & 1.309 / 47.3 & 1.147 / 45.0 & 0.236 / \textbf{60.3} \\
\cmark & Staged  & 0.349 / \textbf{34.8} & 22.33 / 59.3 & \textbf{0.283} / \textbf{91.6} & \textbf{1.224} / \textbf{47.4} & 1.049 / \textbf{45.3} & \textbf{0.131} / 59.8 \\
\cmark & \textbf{Hybrid} (full) & 0.337 / 34.6 & \textbf{16.41} / 59.7 & 0.311 / 89.1 & 1.263 / 46.9 & 1.081 / 45.0 & 0.189 / \textbf{60.3} \\
\bottomrule
\end{tabular}
}
\end{table}
 
\paragraph{SurF-GLQ hyperparameter sweep: protocol.}
The sweep jointly varied: number of Gauss--Legendre points
$Q\in\{4,8,12\}$; number of log-intensity references $J\in\{3,4,5\}$;
conformer pre-encoder (on/off); batch size $\in\{128,256,512\}$;
learning rate $\in[3{\times}10^{-5},\,8{\times}10^{-4}]$; and
pretraining type-loss weight $\beta_{\text{type}}\in[0.51,\,4.28]$. Trials that diverged (negative SurF
loss, indicating numerical failure) were excluded; this affected
$5/105$ runs, all at extreme learning-rate values.
 
\paragraph{Quadrature points and mixture size.}
Table~\ref{tab:ablation_glq} aggregates the sweep by $Q$ and $J$.
 
\begin{table}[ht]
\centering
\setlength{\tabcolsep}{6pt}
\renewcommand{\arraystretch}{0.95}
\caption{SurF-GLQ quadrature ablation: $100$-trial random
hyperparameter sweep varying number of Gauss--Legendre points $Q$ and
log-intensity reference components $J$. We report the best, median, and
trial-level standard deviation of time RMSE, and the best type accuracy
within each group. \textbf{Bold} = best per metric.
$Q{=}8$ matches $Q{=}12$ within trial-level noise
($\sigma_{\text{RMSE}}{\approx}0.05$), validating the default in
Section~\ref{sec:training}.}
\label{tab:ablation_glq}
\footnotesize
\begin{tabular}{cc|ccccc}
\toprule
\textbf{Param} & \textbf{Value} & $n$ & \textbf{RMSE (best)} & \textbf{RMSE (med)} & $\sigma_{\text{RMSE}}$ & \textbf{Acc \% (best)} \\
\midrule
\multirow{3}{*}{$Q$} & $4$  & $9$  & 8.751 & 8.825 & 0.170 & 48.96 \\
                     & $\mathbf{8}$ (default)  & $69$ & 8.729 & \textbf{8.749} & 0.058 & \textbf{49.94} \\
                     & $12$ & $22$ & \textbf{8.725} & 8.786 & 0.035 & 49.49 \\
\midrule
\multirow{3}{*}{$J$} & $3$ & $38$ & \textbf{8.725} & 8.771 & 0.069 & 49.51 \\
                     & $\mathbf{4}$ (default) & $56$ & 8.729 & \textbf{8.748} & 0.053 & \textbf{49.94} \\
                     & $5$ & $6$  & 8.751 & 8.815 & 0.082 & 48.88 \\
\bottomrule
\end{tabular}
\end{table}
 
\paragraph{Conformer $\times$ $Q$ interaction.}
Table~\ref{tab:ablation_conformer_q} cross-tabulates the conformer
toggle against $Q$. The conformer block contributes most when paired
with the recommended $Q{=}8$, where it lifts best-trial type accuracy
from $48.96\%$ ($Q{=}4$, no conformer) to $49.94\%$ while preserving
the lowest median RMSE. Without the conformer, larger $Q$ partially
recovers performance but plateaus by $Q{=}12$.
 
\begin{table}[ht]
\centering
\setlength{\tabcolsep}{6pt}
\renewcommand{\arraystretch}{0.95}
\caption{Conformer $\times$ $Q$ interaction. Conformer + $Q{=}8$
delivers the best accuracy and is competitive in RMSE; further
quadrature points provide no measurable gain.}
\label{tab:ablation_conformer_q}
\footnotesize
\begin{tabular}{cc|cccc}
\toprule
\textbf{Conformer} & $Q$ & $n$ & \textbf{RMSE (best)} & \textbf{RMSE (med)} & \textbf{Acc \% (best)} \\
\midrule
\xmark & $4$  & $9$  & 8.751 & 8.825 & 48.96 \\
\xmark & $12$ & $12$ & 8.725 & 8.761 & 49.38 \\
\cmark & $8$  & $69$ & 8.729 & 8.749 & \textbf{49.94} \\
\cmark & $12$ & $10$ & 8.780 & 8.795 & 49.49 \\
\bottomrule
\end{tabular}
\end{table}
 
\paragraph{Loss-weight sensitivity.}
Table~\ref{tab:ablation_lossweights} bins the trials by
type-loss weight
$\beta_{\text{type}}$. Performance is robust across a wide range:
$\beta_{\text{type}}\in[0.7,1.0]$ forms
a broad plateau within $0.02$ RMSE of the best configuration. Very high type weight ($>2$) over-emphasizes
classification at the expense of timing. The default
$\beta_{\text{type}}{=}0.74$ used in
Section~\ref{sec:experiments} sits inside this plateau.
 
\begin{table}[ht]
\centering
\setlength{\tabcolsep}{6pt}
\renewcommand{\arraystretch}{0.95}
\caption{Loss-weight sensitivity. Mid-range type weight forms
a broad performance plateau.}
\label{tab:ablation_lossweights}
\footnotesize
\begin{tabular}{l|cccc}
\toprule
\textbf{Bin} & $n$ & \textbf{RMSE (best)} & \textbf{RMSE (mean)} & \textbf{Acc \% (best)} \\
\midrule
\multicolumn{5}{l}{\emph{Type weight $\beta_{\text{type}}$}} \\
\quad $<0.7$ & $38$ & \textbf{8.725} & 8.786 & 49.51 \\
\quad $0.7$--$1$ & $46$ & 8.729 & \textbf{8.766} & \textbf{49.94} \\
\quad $1$--$2$ & $10$ & 8.744 & 8.858 & 49.87 \\
\quad $>2$ & $6$ & 8.751 & 8.801 & 49.49 \\
\bottomrule
\end{tabular}
\end{table}
 
\paragraph{Batch-size effect.}
Table~\ref{tab:ablation_batchsize} reports performance by batch size.
Batch $256$ achieves the best mean RMSE ($8.77$) and the highest
best-trial accuracy ($49.94\%$). Batch $512$ matches the lowest
absolute RMSE ($8.725$) but with slightly degraded type accuracy
($49.49\%$); batch $128$ is uniformly worst, consistent with higher
gradient noise at small batches harming both the SurF likelihood and
the classification head.
 
\begin{table}[ht]
\centering
\setlength{\tabcolsep}{6pt}
\renewcommand{\arraystretch}{0.95}
\caption{Batch-size effect on the SurF-GLQ HP sweep.}
\label{tab:ablation_batchsize}
\footnotesize
\begin{tabular}{c|cccc}
\toprule
\textbf{Batch} & $n$ & \textbf{RMSE (best)} & \textbf{RMSE (mean)} & \textbf{Acc \% (best)} \\
\midrule
$128$ & $9$  & 8.751 & 8.887 & 48.96 \\
$\mathbf{256}$ & $69$ & 8.729 & \textbf{8.773} & \textbf{49.94} \\
$512$ & $22$ & \textbf{8.725} & 8.782 & 49.49 \\
\bottomrule
\end{tabular}
\end{table}
 
\paragraph{Top configurations.}
The top-$10$ trials by time RMSE all come from $Q{\in}\{8,12\}$ ($60\%$
$Q{=}8$, $40\%$ $Q{=}12$) and $J{\in}\{3,4\}$, span both conformer
settings ($60\%$ on, $40\%$ off), and use either batch $256$ or $512$.
None of the top-$10$ trials uses $Q{=}4$, $J{=}5$, or batch $128$,
indicating these are the only design choices the sweep clearly rules
out. The fact that several near-optimal HP regions exist---rather than
a single brittle peak---supports the broader claim that the SurF
training objective in \eqref{eq:amort_loss} is well-conditioned.
 
\section{Additional Theoretical Material}

\subsection{Joint Modeling of Times and Marks}\label{app:marks}

For $K$ event types, SurF factorizes the joint conditional intensity as
\begin{equation}\label{eq:mark_factorization}
\lambda^*(t, k \mid \mathcal{H}_t)
= \lambda^*(t \mid \mathcal{H}_t)\cdot p(k \mid t, \mathcal{H}_t),
\qquad \sum_{k=1}^K p(k \mid t, \mathcal{H}_t) = 1,
\end{equation}
where $\lambda^*(t \mid \mathcal{H}_t)$ is the marginal ground intensity
(modeled by $\Lambda_\theta$) and $p(k \mid t, \mathcal{H}_t)$ is the
conditional mark distribution (modeled by $\text{TypeNet}_\theta$ with a
softmax head). Substituting into the point-process likelihood, the joint
NLL decomposes exactly:
\begin{equation}\label{eq:joint_nll}
-\log p(\mathcal{D})
= \underbrace{\mathcal{L}_{\text{SurF}}(\theta)}_{\text{time}}
+ \underbrace{\sum_{i=1}^N -\log p_\theta(k_i \mid t_i, \mathcal{H}_{t_i})}_{\mathcal{L}_{\text{type}}(\theta)}.
\end{equation}
No variational bound or coupling term is required; the factorization is
exact at the likelihood level. The weight $\beta_{\text{type}}$ in
Algorithm~\ref{alg:surf_training} is introduced purely to balance the
relative scales of the two NLL components during optimization, not as a
variational coefficient.

\paragraph{Preservation of the TRT under marking.} Because
$\lambda^*(t,k) = \lambda^*(t)\,p(k\mid t)$ is a strict product, the TRT
applied to the ground intensity $\lambda^*(t) = \sum_k \lambda^*(t,k)$
yields the same $\mathrm{Exp}(1)$ rescaling regardless of $K$, so
Theorem~\ref{thm:reverse_rescaling} holds unchanged and
Proposition~\ref{prop:domain_invariant} is unaffected by the number of
types. Marks are generated conditionally at sampling time
(Algorithm~\ref{alg:surf_sampling_newton}).

\paragraph{Alternative: $K$-dimensional intensity.} A per-type
formulation $\{\lambda_k^*(t)\}_{k=1}^K$ with separate compensators is
also supported by letting the cumulative-intensity head output $K$
copies of the parameters. The TRT then applies per-type, giving $K$
independent $\mathrm{Exp}(1)$ rescalings; the joint likelihood equals
the sum of $K$ per-type instances of the unified loss.

\subsection{Relationship to Prior Cumulative-Hazard Neural TPPs}\label{app:related_cumhazard}

Parameterizing the cumulative hazard $\Lambda_\theta$ directly and
recovering $\lambda_\theta$ via differentiation is not new: Omi et
al.~\citep{omi2019fully} introduced this construction in \emph{FullyNN},
where $\Lambda_\theta$ is a feed-forward network with non-negative
weights, and Shchur et al.~\citep{shchur2019intensity} use a log-normal
mixture with a related closed-form compensator. SurF differs along three
axes.

\paragraph{(i) Conceptual framing.} Prior work treats the monotone
network as an engineering convenience for avoiding the
$\int_0^T \lambda\,ds$ integral. SurF treats $\Lambda_\theta$ as a
normalizing flow between the point process and i.i.d.\ $\mathrm{Exp}(1)$
noise, with the TRT supplying the Jacobian. This licenses the
dataset-invariance claim (Proposition~\ref{prop:domain_invariant}) and
motivates cross-dataset pretraining; FullyNN and its descendants are
trained per-dataset and make no such claim.

\paragraph{(ii) Parameterization.} FullyNN enforces monotonicity via non-negative-weight MLPs, which are ill-conditioned at depth. SurF-CSB uses a convex combination of softplus basis functions that admits closed-form derivatives. SurF-MoE specializes further to a mixture of exponentials with a fully closed-form loss; SurF-GLQ provides an unconstrained-MLP alternative with per-interval quadrature.

\paragraph{(iii) Amortization.} FullyNN's hazard depends on history
through an RNN hidden state shared across the monotone MLP input. SurF
amortizes via a Transformer encoder whose output $\mathbf{h}_{i-1}$
parameterizes a fresh $\Lambda_\theta(\cdot \mid \mathbf{h}_{i-1})$ at
each interval, yielding the diagonal Jacobian of
Appendix~\ref{app:likelihood-derivation} and the $O(N)$ gradient bound
of Appendix~\ref{app:surf-grad} rather than the $O(\sigma^N)$ scaling of
RNN-based hazards.

\paragraph{Empirical comparison.} Table~\ref{tab:nextevent} includes
FullyNN as a baseline; SurF outperforms it on all six benchmarks, and
the gap widens in the zero-shot regime where FullyNN has no mechanism
for cross-dataset transfer.

\subsection{When to Use Which Variant, and the Role of Quadrature}\label{app:variants_quadrature}

The three variants target different regimes and differ in how (if at
all) they invoke quadrature.

\paragraph{SurF-MoE --- closed-form, fastest.} Between-event intensity is
monotone-decaying and multi-scale. Compensator and derivative are exact
closed forms; the loss invokes zero quadrature. This is the default for
foundation-model pretraining.

\paragraph{SurF-CSB --- closed-form, most expressive.} Intensity is expressive (delayed excitation, hump-shaped profiles); the shifted-softplus basis keeps both $\Lambda_\theta$ and $\lambda_\theta$ in closed form.

\paragraph{SurF-GLQ with fixed-cost per-interval quadrature.} The
intensity is an unconstrained positive MLP; $\Lambda_\theta$ uses $Q=8$
Gauss--Legendre nodes \emph{per inter-event interval} $[0, \tau_i]$, not
over the full window $[0, T]$.

\begin{table}[ht]
\centering
\small
\caption{Quadrature cost across methods. Only SurF-GLQ uses quadrature;
unlike THP/SAHP its cost is bounded per interval, independent of $T$,
and the error does not enter the gradient in practice
($<\!10^{-12}$ at $Q=8$).}
\label{tab:quadrature_compare}
\begin{tabular}{lcccc}
\toprule
\textbf{Method} & \textbf{Quadrature cost} & \textbf{Grows with $T$?} & \textbf{In gradient?} & \textbf{Exact?} \\
\midrule
THP / SAHP & $O(NQ')$, $Q'\!\propto\!T$ & yes & yes (biased) & no \\
NHP& $O(NM)$, $M\!\gg\!1$ & yes & yes (stochastic) & no \\
SurF-MoE & $0$ & --- & no & yes \\
SurF-CSB & $0$ & --- & no & yes \\
SurF-GLQ & $O(NQ)$, $Q\!=\!8$ fixed & \textbf{no} & negligible ($<\!10^{-12}$) & yes up to $10^{-12}$ \\
\bottomrule
\end{tabular}
\end{table}

\paragraph{Reading the ``no quadrature'' claim precisely.} The statement
``exact likelihoods with no numerical quadrature'' in the introduction
should be read as: \emph{no quadrature at all for MoE and CSB; fixed
$O(Q)$-per-interval quadrature for GLQ, with $T$-independent cost and
negligible error that does not enter the gradient.} In all three cases,
likelihood evaluation is non-iterative and the training signal is
unbiased to machine precision.

\subsection{Shared Canonical Target}

\begin{proposition}[Shared canonical target]\label{prop:domain_invariant}
Let $\mathcal{P}_1,\mathcal{P}_2$ be point processes arising from potentially
different datasets with intensities $\lambda_1^*,\lambda_2^*$, each satisfying
the hypotheses of Theorem~\ref{thm:reverse_rescaling}, and let
$(\phi_\theta,\Lambda_\theta)$ be trained jointly on the mixture
$\pi_1\mathcal{P}_1+\pi_2\mathcal{P}_2$ with $\pi_d>0$ by minimizing the
population SurF loss. Assume realizability: some $\theta^\star$ satisfies
$\Lambda_{\theta^\star}(\tau\mid\phi_{\theta^\star}(\mathcal{H}))
=\Lambda_d^*(\tau\mid\mathcal{H})$ for $\mathcal{P}_d$-almost every history
$\mathcal{H}$ and both $d$. Then every population minimizer agrees with
$\Lambda_d^*$ $\mathcal{P}_d$-almost everywhere for both $d$, and the rescaled
gaps $\Lambda_\theta(\tau_i\mid\phi_\theta(\mathcal{H}_{t_i}))$ are i.i.d.\
$\mathrm{Exp}(1)$ on each dataset. The two datasets therefore share one
target, and the loss couples the encoder through it even though its
$(t_i,k_i)$ inputs remain dataset-specific.
\end{proposition}

\begin{proof}
By Appendix~\ref{app:likelihood-derivation}, $\mathcal{L}_{\text{SurF}}(\theta)$
evaluated on a sequence drawn from $\mathcal{P}_d$ is the negative
log-likelihood of that sequence under the model, so writing $p_d^*$ for the
true inter-arrival law of $\mathcal{P}_d$ and $p_\theta$ for the model's, the
population objective is
\begin{equation}
R(\theta) = \sum_{d=1,2}\pi_d\,\mathbb{E}_{p_d^*}\bigl[-\log p_\theta\bigr]
= \sum_{d=1,2}\pi_d\,\mathrm{KL}\bigl(p_d^*\,\|\,p_\theta\bigr) + c,
\end{equation}
where $c$ collects the entropies of $p_1^*$ and $p_2^*$ and does not
depend on $\theta$. Each summand is minimized
exactly when $\mathrm{KL}(p_d^*\|p_\theta)=0$, i.e.\ $p_\theta=p_d^*$
$\mathcal{P}_d$-almost everywhere, and realizability makes the two minima
simultaneously attainable; since $\pi_d>0$, $R$ is therefore minimized
precisely at those $\theta$ that match both. A point process law determines
its compensator and conversely, and
$\lambda^{\dagger}_\theta=\partial\Lambda_\theta/\partial(\Delta t)$
by~\eqref{eq:amort_intensity}, so this is equivalent to
$\Lambda_\theta(\cdot\mid\phi_\theta(\cdot))=\Lambda_d^*$
$\mathcal{P}_d$-almost everywhere. Applying
Corollary~\ref{cor:time_rescaling_corollary} to each $\mathcal{P}_d$ gives the
$\mathrm{Exp}(1)$ conclusion.
\end{proof}

\section{Datasets}\label{app:datasets}

\begin{table}[ht]
\centering
\caption{Pre-training corpus composition. The corpus spans 6 datasets with
diverse temporal scales, totaling ${\sim}1.3$M events across ${\sim}32$K
sequences.}
\label{tab:pretraining_corpus}
\begin{small}
\begin{tabular}{l|rrrrl}
\toprule
\textbf{Dataset} & \textbf{Sequences} & \textbf{Events} & \textbf{Event Types} & \textbf{Avg.\ Length} & \textbf{Temporal Scale} \\
\midrule
Taxi & 2,000 & 74,078 & 10 & 37.04 & minutes--hours \\
Earthquake & 4,300 & 70,723 & 7 & 16.46 & hours--days \\
StackOverflow & 2,203 & 142,777 & 22 & 64.81 & hours--days \\
Amazon & 9,227 & 413,420 & 16 & 44.81 & hours--days \\
Retweet & 12,055 & 493,708 & 3 & 40.95 & seconds--hours \\
Taobao & 2,000 & 115,397 & 17 & 57.70 & hours--days \\
\midrule
\textbf{Total} & \textbf{31,785} & \textbf{1,310,103} & 75 & \textbf{41.22} & seconds--days \\
\bottomrule
\end{tabular}
\end{small}
\end{table}

We evaluate SurF on six real-world datasets spanning e-commerce,
transportation, social media, geoscience, and online collaboration:
\begin{itemize}[leftmargin=*]
\item \textbf{Taobao}~\citep{xue2022hypro}: user click sequences from a
large-scale online recommendation platform, capturing temporal patterns
of user engagement with product categories.
\item \textbf{Taxi}~\citep{whong-14-taxi}: taxi pick-up events across
New York City neighborhoods, providing insights into urban mobility
patterns.
\item \textbf{Retweet}~\citep{zhou2013learning}: social cascades of
retweet behaviors and user identities over time.
\item \textbf{StackOverflow}~\citep{snapnets}: activity logs from a
question-and-answer platform with events representing posting and voting
interactions.
\item \textbf{Amazon}~\citep{amazon-2018}: product review sequences
where each event is a user reviewing a product in a specific category.
\item \textbf{Earthquake}: seismic event sequences with magnitude and
location codes, covering multi-day inter-event gaps from regional
geological catalogs.
\end{itemize}

\section{Model Architecture and Hyperparameters}\label{app:architecture}

The model has two components: a history encoder
$\phi_\theta\!:\mathbb{R}^+\!\times\!\mathcal{H}\!\to\mathbb{R}^d$ producing
the encoding $\mathbf{h}_{i-1}$, and a cumulative-intensity head that maps
$\mathbf{h}_{i-1}$ to $\Lambda_\theta(\cdot\mid\mathbf{h}_{i-1})$
(Appendix~\ref{app:cumint_arch}).

\subsection{History Encoding}

A multi-layer Transformer with time-aware embeddings:
\begin{equation}
\phi_\theta(t, \mathcal{H}_t) = \text{TransformerEncoder}_\theta(\textrm{Conformer}_\xi(\mathbf{E}_{\text{combined}}, \mathbf{M})),
\end{equation}
\begin{equation}
\mathbf{E}_{\text{combined}} = [\mathbf{E}_{\text{type}}; \mathbf{E}_{\text{time}}] \in \mathbb{R}^{N \times 2d_{\text{hidden}}}.
\end{equation}
Type embeddings: $\mathbf{E}_{\text{type}}^{(i)} = \text{Embedding}(k_i + 1) \in \mathbb{R}^{d_{\text{hidden}}}$
(the offset ensures valid indices).

\textbf{Temporal encoding.} Normalized time prevents numerical
instability:
\begin{equation}
\mathbf{E}_{\text{time}}^{(i)} = \text{MLP}_{\text{time}}(t_i / T_{\text{scale}}),
\qquad T_{\text{scale}} = \max(t_1, \ldots, t_N) + \epsilon,
\end{equation}
with
$\text{MLP}_{\text{time}}(x) = \text{Linear}_2(\tanh(\text{Linear}_1(x)))$,
$\text{Linear}_1\!:\mathbb{R}\!\to\mathbb{R}^{d_{\text{hidden}}/2}$ and
$\text{Linear}_2\!:\mathbb{R}^{d_{\text{hidden}}/2}\!\to\mathbb{R}^{d_{\text{hidden}}}$.
The Conformer may or may not be included, depending on the SurF variant. The Transformer is causally masked.

\subsection{Cumulative-Intensity Head}\label{app:cumint_arch}

For the amortized variants, the Transformer output $\mathbf{h}_{i-1}$
feeds a cumulative-intensity head that produces the parameters of
$\Lambda_\theta(\cdot \mid \mathbf{h}_{i-1})$.

\textbf{SurF-CSB head.} A single linear layer maps $\mathbf{h}_{i-1}$ to
$3M$ raw parameters $(\hat\alpha_m, \hat\beta_m, \hat\delta_m)_{m=1}^M$,
with $\alpha_m = \text{softplus}(\hat\alpha_m)$,
$\beta_m = \text{softplus}(\hat\beta_m) + \epsilon$. The $\delta_m$ are
unconstrained, initialized uniformly in $[-3, 3]$ to span early-to-late
activation.

\textbf{SurF-GLQ head.} A two-layer MLP takes $(\Delta t, \mathbf{h}_{i-1})$
and outputs a scalar,
$f_\theta(\Delta t, \mathbf{h}) = \text{MLP}([\Delta t/T_{\text{scale}}; \mathbf{h}])$.
Intensity:
$\tilde\lambda_\theta = \text{softplus}(f_\theta)$; cumulative via
the Gauss--Legendre quadrature of \eqref{eq:glq}. The MLP is lightweight
(two layers, $d_{\text{hidden}}/2$ width), and $Q$ passes are batched.

\textbf{SurF-MoE head.} A single linear layer maps $\mathbf{h}_{i-1}$ to
$2J$ raw parameters $(\hat w_j, \hat\gamma_j)_{j=1}^J$. Rate biases are
initialized spanning $[-2, 4]$ for multi-scale coverage.

\subsection{Type Prediction Head}

The output from the history encoding is passed through a Type Refinement Tower consisting of 3-4 Transformer encoder layers to produce a type context. This context is passed through an MLP before being projected to the type embedding space.
\begin{align}
    \text{TypeNet}_\theta(\mathbf{h}) = \text{Proj}(\text{MLP}_\text{type}(\text{TypeRefinementTower}(\phi_\theta(t, \mathcal{H}_t))))
\end{align}

\subsection{Hyperparameter Search}

We optimize the time RMSE directly,
\begin{equation}
\mathcal{L}_{\text{hyperopt}} = \text{RMSE}_{\text{time}} 
\end{equation}
since we found that type accuracy is not heavily impacted by small changes in hyperparameters. We use Optuna with TPE sampling and median pruning for early termination
of unpromising trials (patience $10$ epochs).

\textbf{Searched:} $d_{\text{hidden}} \in \{32, 64, 128, 256, 512\}$,
layers $L \in \{2, 3, 4\}$, batch size $B \in \{128, 256, 512\}$,
learning rate $\eta \in [10^{-5}, 10^{-3}]$ (log-uniform), basis
components $M \in \{8, 12, 16\}$ (CSB) or quadrature points
$Q \in \{4, 8, 12\}$ (GLQ).

\textbf{Fixed:} $H = 4$ attention heads, dropout $0.1$, weight decay
$10^{-5}$, cosine annealing with warm restarts, gradient accumulation
$1$. We require $d_{\text{hidden}} \equiv 0 \pmod H$.

\section{Computational Complexity}\label{app:complexity}

\textbf{Training per sequence.}
\begin{itemize}[leftmargin=*]
\item RNN TPPs: $O(N \cdot H \cdot D)$ due to sequential dependencies
and BPTT.
\item SurF-MoE: $O(N \cdot J)$ with exact closed-form integration.
\item SurF-CSB: $O(N \cdot M)$ with closed-form derivative evaluation.
\item SurF-GLQ: $O(N \cdot Q \cdot D_f)$ where $D_f$ is the MLP cost,
with parallelizable quadrature points.
\end{itemize}

\textbf{Sampling.}
\begin{itemize}[leftmargin=*]
\item RNN TPPs: $O(N \cdot H \cdot D)$ sequential operations.
\item All SurF variants: $O(N \cdot S)$ with $S \leq 8$ Newton
iterations, dominated by the Transformer forward pass.
\end{itemize}

The ability to parallelize training and the stable gradient flow make
SurF particularly suitable for long sequences with varying time scales,
where RNN methods struggle with both efficiency and numerical stability.

\section{Additional Empirical Analysis}

\subsection{Finetuned Models vs Top Baselines}

\begin{figure}[ht]
\vspace{-0.3em}
\centering
\includegraphics[width=\textwidth]{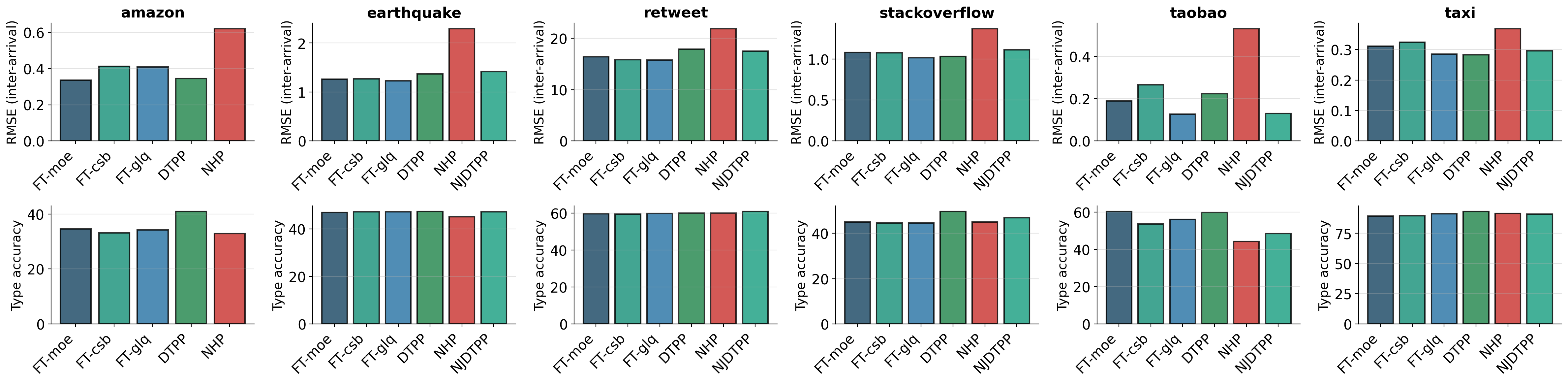}
\vspace{-1.8em}
\caption{\small Next-event inter-arrival RMSE (top) and type accuracy
(bottom) for finetuned SurF variants
and three baseline models (DTPP, NHP, NJDTPP). NJDTPP was not evaluated for Amazon in their paper, and is hence omitted.}
\label{fig:onestep}
\vspace{-0.6em}
\end{figure}

\subsection{Forecast Trajectories}
\begin{figure}[t]
    \centering
    \includegraphics[width=\textwidth]{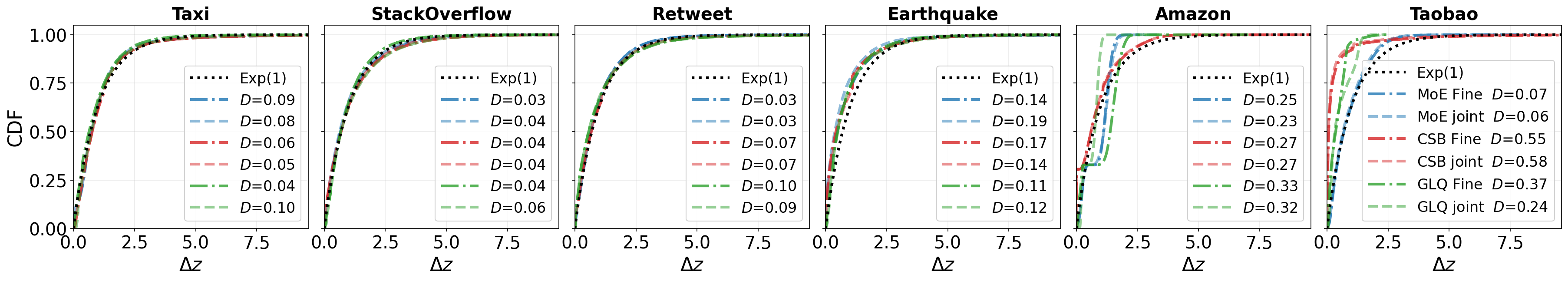}
    \caption{
        \textbf{Time-rescaling goodness-of-fit across datasets and variants.}
        Empirical CDF of the residuals
        $\Delta z_i = \Lambda(\tau_i \mid h_{i-1})$ on the test split.
        A perfectly calibrated model produces residuals that are i.i.d.\
        $\mathrm{Exp}(1)$ (black dotted curve); smaller KS statistic
        $D$ indicates better calibration. SurF achieves near-ideal
        calibration on Taxi and StackOverflow ($D \leq 0.1$ for all
        variants).
    }
    \label{fig:calibration_ks_cdf_full}
\end{figure}
\begin{figure}[htbp]
\centering
\includegraphics[width=\textwidth]{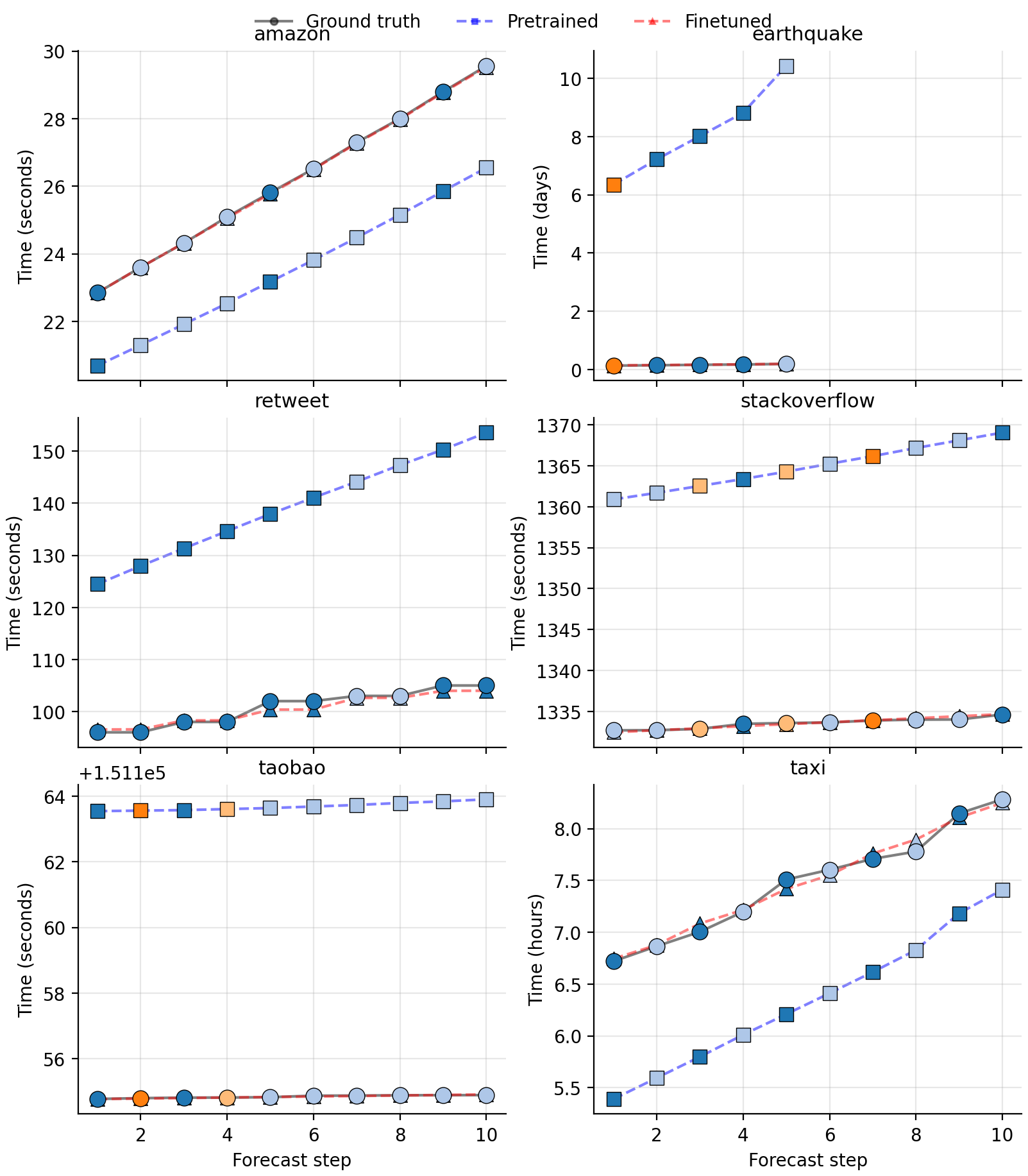}
\caption{Forecast trajectories across datasets. Each subplot shows one
representative sequence per dataset (chosen by lowest finetuned RMSE).
Ground truth (circles), pretrained (squares), and finetuned (triangles)
forecast times; colors indicate event types.}
\label{fig:trajectories_all_datasets}
\end{figure}

\subsection{Next-event prediction of finetuned baselines}

\begin{table}[ht]
\vspace{-1 em}
\centering
\setlength{\tabcolsep}{3pt}
\renewcommand{\arraystretch}{0.95}
\caption{Next-event prediction: time RMSE / type accuracy (\%).
\textbf{Bold} = best, \underline{underline} = second-best. ``--'' = not
reported. $^{\dagger}$ DTPP conditions its mark prediction on the
ground-truth next-event time, which biases its type accuracy upward.}
\label{tab:nexteventbaselines}
\resizebox{\columnwidth}{!}{%
\footnotesize
\begin{tabular}{l|cccccc}
\toprule
\textbf{Model} & \textbf{Amazon} & \textbf{Retweet} & \textbf{Taxi} & \textbf{Earthquake} & \textbf{StackOverflow} & \textbf{Taobao} \\
\midrule
\textsc{AttNHP} (joint) & {2.938} / 27.0 & 19.85 / 58.6 & 0.409 / 46.5 & 1.510 / 25.4 & 2.022 / 43.2 & 1.373 / {46.6} \\
\textsc{THP} (joint) & {1.018} / 30.0 & 23.22 / 55.1 & 0.344 / 44.3 & 1.451 / \textbf{47.2} & 1.231 / 42.1 & 0.260 / {43.6} \\
\textsc{DTPP} (joint) & \textbf{0.341} / \textbf{39.5}$^{\dagger}$ & 17.89 / \textbf{59.3}$^{\dagger}$ & \textbf{0.287} / \textbf{92.0}$^{\dagger}$ & \underline{1.372} / \underline{47.1}$^{\dagger}$ & \textbf{1.069} / \textbf{47.9}$^{\dagger}$ & \underline{0.217} / \underline{54.8}$^{\dagger}$ \\
\textsc{SurF-MoE} (joint) & \underline{0.343} / \underline{34.4} & 16.04 / \underline{59.2} & 0.491 / 89.7 & 1.541 / \underline{47.1} & 1.255 / \underline{44.8} & 0.231 / \textbf{60.0} \\
\textsc{SurF-CSB} (joint) & 0.410 / 30.4 & \underline{15.89} / 56.8 & 0.456 / \underline{90.3} & \textbf{1.294} / \underline{47.1} & 1.104 / 44.5 & 0.369 / 51.5 \\
\textsc{SurF-GLQ} (joint) & 0.477 / 33.3 & \textbf{15.80} / 59.1 & \underline{0.330} / 88.2 & \textbf{1.294} / \textbf{47.2} & \underline{1.098} / 44.4 & \textbf{0.126} / 53.4 \\
\midrule
\textsc{AttNHP} (fine) & 1.314 / 23.7 & 20.84 / 59.3 & 0.699 / 64.3 & 1.593 / 46.3 & 1.311 / 31.5 & 3.654 / 43.6 \\
\textsc{THP} (fine) & 0.564 / \underline{35.2} & 22.91 / 59.3 & 0.381 / \underline{91.3} & 1.783 / \underline{47.2} & 1.378 / \underline{45.1} & 0.276 / 53.2 \\
\textsc{DTPP} (fine) & \underline{0.341} / \textbf{40.0}$^{\dagger}$ & 17.89 / 59.2$^{\dagger}$ & \underline{0.287} / \textbf{92.3}$^{\dagger}$ & 1.372 / 46.9$^{\dagger}$ & 1.069 / \textbf{48.2}$^{\dagger}$ & 0.217 / 54.8$^{\dagger}$ \\
\textsc{SurF-MoE} (fine)  & \textbf{0.336} / 34.7 & 16.23 / \underline{59.7} & 0.311 / 89.5 & 1.258 / 47.0 & 1.077 / \underline{45.1} & \underline{0.189} / \textbf{60.3} \\
\textsc{SurF-CSB} (fine)  & 0.409 / 33.1 & \underline{15.81} / 59.5 & 0.288 / 90.5 & \underline{1.256} / \underline{47.2} & \underline{1.039} / 44.7 & 0.224 / \underline{57.2} \\
\textsc{SurF-GLQ} (fine)  & 0.409 / 34.2 & \textbf{15.78} / \textbf{59.8} & \textbf{0.286} / 90.9 & \textbf{1.228} / \textbf{47.3} & \textbf{1.018} / 44.6 & \textbf{0.126} / 56.2 \\
\bottomrule
\end{tabular}
}
\end{table}

\subsection{Transfer Cost Under Leave-One-Out}
\label{sec:transfer_cost}
A natural question is whether any neural TPP trained on the same five corpora would transfer to the held-out sixth about as well as SurF does. To answer it we measure each model's leave-one-out RMSE against its \emph{own} in-corpus RMSE (for SurF the jointly-trained checkpoint, for each baseline its per-dataset model), so that every model is judged against its own achieved level rather than against ours, and we report the degradation each suffers when the target dataset is removed from training (Table~\ref{tab:loo_degradation}).

\begin{table}[ht]
\centering\small
\caption{\textbf{Transfer cost under leave-one-out.} Lower is better; $<\!1\times$ means the held-out checkpoint is better than the in-corpus one.}
\label{tab:loo_degradation}
\begin{tabular}{l|cc}
\toprule
\textbf{Model} & \textbf{Median degradation} & \textbf{Worst-case degradation} \\
\midrule
DTPP~\citep{panos2024decomposable}   & $1.39\times$ & $34.2\times$ (Amazon: $0.332\!\to\!11.35$) \\
THP~\citep{zuo2020transformer}       & $5.23\times$ & $13.9\times$ (Amazon: $0.657\!\to\!9.127$) \\
AttNHP~\citep{yang-2022-transformer} & $1.42\times$ & $3.0\times$  (Amazon: $0.657\!\to\!1.975$) \\
\midrule
\textsc{SurF-CSB} (ours) & $\mathbf{0.97\times}$ & $\mathbf{2.4\times}$ (StackOverflow) \\
\bottomrule
\end{tabular}
\end{table}

SurF-CSB stays within $1.3\times$ of its own in-corpus performance on five of six datasets, and on Amazon, Earthquake and Taobao it is better zero-shot than jointly trained. This is not particular to CSB: worst-case degradation is $2.5\times$ for MoE and $1.8\times$ for GLQ. DTPP and THP, by contrast, collapse. DTPP performs well in-corpus but not once the target is removed: its Amazon error grows by a factor of thirty-four, to $11.35$, roughly seventeen times the weakest in-corpus number in the same block of Table~\ref{tab:multihorizon} ($0.657$, THP and AttNHP). THP performs poorly in-corpus and transfers poorly as well. AttNHP transfers well in relative terms ($1.42\times$ median, $3.0\times$ worst) and is not far from CSB, but in absolute terms it is weaker than the best SurF variant on five of six datasets ( $6.0\times$ on Amazon, $3.9\times$ on Taobao, $2.3\times$ on Taxi ) with StackOverflow the sole exception (AttNHP better by $1.5\times$).

\end{document}